\documentclass[conference]{IEEEtran}
\usepackage{times}

\usepackage[numbers]{natbib}
\usepackage{multicol}
\usepackage[bookmarks=true]{hyperref}
\usepackage{amsmath}
\usepackage{algorithm}
\usepackage{algpseudocode}
\usepackage{amsthm}
\usepackage{xcolor}
\usepackage{booktabs}
\usepackage{multirow}
\usepackage{graphicx}
\usepackage{amssymb}
\usepackage{grffile}
\usepackage{cuted}
\usepackage{caption}
\usepackage{wrapfig}
\usepackage{subcaption} % Only needed for Option 3 (Side-by-Side)
\usepackage{placeins}
\usepackage{float}

\usepackage{titlesec}
\usepackage{balance}

\titlespacing*{\section}{0pt}{1.5ex plus 0.5ex minus 0.2ex}{1ex plus 0.2ex}
\titlespacing*{\subsection}{0pt}{1ex plus 0.3ex minus 0.2ex}{0.5ex plus 0.1ex}
\titlespacing*{\subsubsection}{0pt}{0.8ex plus 0.2ex minus 0.1ex}{0.4ex plus 0.1ex}
\newcommand{\new}[1]{#1}
\newtheorem{proposition}{Proposition}

\begin{document}

% paper title
\title{TAIL-Safe: Task-Agnostic Safety Monitoring for Imitation Learning Policies}

% You will get a Paper-ID when submitting a pdf file to the conference system
\author{\authorblockN{Riad Ahmed \quad Momotaz Begum}\\
\authorblockA{University of New Hampshire\\
Email: \{riad.ahmed, momotaz.begum\}@unh.edu\\
Project page: \url{https://assistiveroboticsunh.github.io/tail_safe/}}}

%\author{\authorblockN{Michael Shell}
%\authorblockA{School of Electrical and\\Computer Engineering\\
%Georgia Institute of Technology\\
%Atlanta, Georgia 30332--0250\\
%Email: mshell@ece.gatech.edu}
%\and
%\authorblockN{Homer Simpson}
%\authorblockA{Twentieth Century Fox\\
%Springfield, USA\\
%Email: homer@thesimpsons.com}
%\and
%\authorblockN{James Kirk\\ and Montgomery Scott}
%\authorblockA{Starfleet Academy\\
%San Francisco, California 96678-2391\\
%Telephone: (800) 555--1212\\
%Fax: (888) 555--1212}}

% avoiding spaces at the end of the author lines is not a problem with
% conference papers because we don't use \thanks or \IEEEmembership

% for over three affiliations, or if they all won't fit within the width
% of the page, use this alternative format:
% 
%\author{\authorblockN{Michael Shell\authorrefmark{1},
%Homer Simpson\authorrefmark{2},
%James Kirk\authorrefmark{3}, 
%Montgomery Scott\authorrefmark{3} and
%Eldon Tyrell\authorrefmark{4}}
%\authorblockA{\authorrefmark{1}School of Electrical and Computer Engineering\\
%Georgia Institute of Technology,
%Atlanta, Georgia 30332--0250\\ Email: mshell@ece.gatech.edu}
%\authorblockA{\authorrefmark{2}Twentieth Century Fox, Springfield, USA\\
%Email: homer@thesimpsons.com}
%\authorblockA{\authorrefmark{3}Starfleet Academy, San Francisco, California 96678-2391\\
%Telephone: (800) 555--1212, Fax: (888) 555--1212}
%\authorblockA{\authorrefmark{4}Tyrell Inc., 123 Replicant Street, Los Angeles, California 90210--4321}}

\maketitle

\begin{strip}
    \centering
    \includegraphics[width=0.9\textwidth]{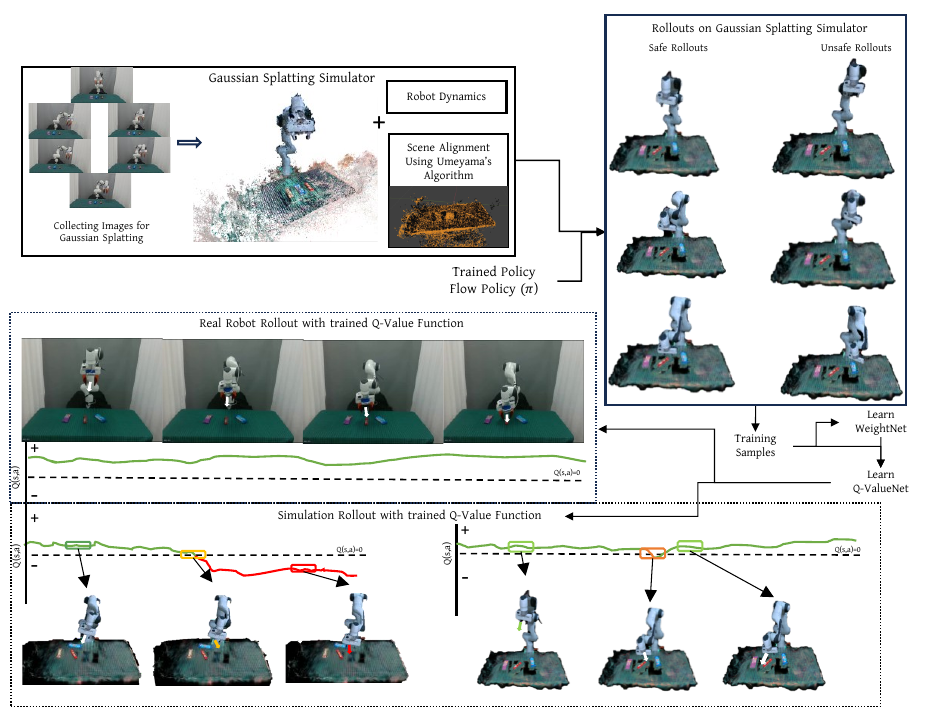}
    \captionof{figure}{\textbf{Overview of TAIL-Safe.} \textbf{(Top-left)} A Gaussian Splatting pipeline constructs a digital twin ($\sim$20 min: 5 min capture + 15 min reconstruction). \textbf{(Top-right)} The simulator generates safe/unsafe trajectories under perturbations for training WeightNet (score fusion) and Q-ValueNet (success prediction). \textbf{(Middle)} At deployment, TAIL-Safe monitors $Q(s,a)$ in real time, remaining inactive while $Q(s,a) > 0$. \textbf{(Bottom)} When $Q(s,a)$ approaches zero, recovery steers the system to safety; without intervention, the policy fails.}
    \label{fig:teaser}
\end{strip}

% \begin{figure*}[t]
%   \centering
%   \includegraphics[width=\textwidth]{images/overall_diagram.pdf} % Use 1.0\textwidth for full width
%   \caption{Overall System}
%   \label{fig:teaser}
% \end{figure*}
\begin{abstract}
Recent imitation learning (IL) algorithms -- such as flow-matching~\cite{lipman2022flow} and diffusion~\cite{chi2025diffusion} policies -- demonstrate remarkable performance in learning complex manipulation tasks. However, these policies often fail even when operating within their training distribution due to extreme sensitivity to initial conditions and irreducible approximation errors that lead to compounding drift. This makes it unsafe to deploy IL policies in the field where out-of-distribution scenarios are prevalent. A prerequisite for safe deployment is enabling the policy to determine whether it can execute a task the way it was learned from demonstrations. This paper presents TAIL-Safe, a principled approach to identify, for a trained IL policy, a \textit{safe set} from where the policy empirically succeeds in completing the learned task. We propose a Lipschitz-continuous Q-value function that maps state-action pairs to a long-term safety score based on three short-term task-agnostic criteria: visibility, recognizability, and graspability. The zero-superlevel set of this function characterizes an empirical control invariant set over state-action pairs. When the nominal policy proposes an action outside this set, we apply a recovery mechanism inspired by Nagumo's theorem that uses gradient ascent on the Q-function to steer the policy back to safety. To learn this Q-function, we construct a high-fidelity digital twin using Gaussian Splatting that enables systematic collection of failure data without risk to physical hardware. Experiments with a Franka Emika robot demonstrate that flow-matching policies, which fail under run-time perturbations, achieve consistent task success when guided by the proposed TAIL-Safe.
\end{abstract}

\IEEEpeerreviewmaketitle

\section{Introduction}
Imitation learning (IL) research has traditionally focused on task-success to measure policy performance. Recent IL algorithms -- such as flow matching \cite{lipman2022flow}, diffusion \cite{chi2025diffusion}, and variants -- demonstrate remarkable task-success in complex manipulation tasks \cite{zhao2023aloha, ze20243d, black2024pi0}. Although these policies can fail miserably, discussion on policy failures is nascent in the current IL literature. Run-time variations are a known cause of failure for visuomotor policies \cite{ross2011dagger,codevilla2019exploring}. Even deterministic policies can behave unpredictably under minor changes in initial conditions, and compounding errors can cause any BC-based policy to fail even in-distribution. Therefore, a prerequisite for IL policies to be employed ubiquitously in field robotics -- especially in domains where policy failure may endanger human lives or properties -- is offering a basic safety assurance: \textit{the task will be executed at run time the way it was learned from demonstrations}. This paper refers to this binary quantity as \textit{task success assurance}. For example, with \textit{task success assurance}, picking up a glass will \underline{never} result in breakage due to grasp failure, or picking up medication will \underline{always} succeed if placed in designated locations. Without \textit{task success assurance}, an IL policy should not roll out. Evaluating \textit{task success assurance} requires the policy to know: 1) \textit{\textbf{what}} ways it can fail, 2) \textit{\textbf{where}} in state space failures may occur, and 3) \textit{\textbf{how}} to avoid states/actions leading to failure. Knowing the \textit{what, where,} and \textit{how} holistically is highly non-trivial. For example, knowledge of how a task may fail must be either explicitly supplied by humans \cite{kelly2019hgdagger}, limiting generalizability, or extracted from failed demonstrations which may be unavailable due to safety concerns. Identifying the area in state space from where the policy may fail -- complementary to where it is assured to succeed -- is related to finding the control invariant set.  Mathematical tools in classical control theory -- such as, control barrier functions~\cite{ames2017cbf, blanchini1999set} and reachability-based methods~\cite{bansal2017hamilton} -- are capable of identifying invariant sets for dynamic systems with known system dynamics. System dynamics however is unpredictable in visuomotor IL. Additionally, all these methods are difficult to scale up in high dimension and visuomotor IL is a high dimensional domain. These make it difficult to directly adopt control theoretic approaches for safety in IL. If the control invariant set of a policy can be identified, there are principled approaches, such as Nagumo's tangentiality condition \cite{nagumo1942lage,blanchini2008set}, that can be leveraged to mitigate policy failure. Overall, safety at deployment is a nascent area in IL research. There is no current work that enables an IL policy to offer \textit{task success assurance} at run-time in a task agnostic manner. This paper bridges that gap. 
\par The proposed framework, termed \textbf{T}ask-\textbf{A}gnostic \textbf{I}mitation \textbf{L}earning \textbf{Safe}ty Watchdog (TAIL-Safe), enables a trained deterministic policy to identify the area in state space from where the training data indicates the policy can successfully complete the task (Figure~\ref{fig:teaser}). This empirical assurance holds when the run-time environment remains within the distribution captured by the digital twin---specifically, when object positions vary within $\pm 5$cm, orientations within $\pm 30^\circ$, and lighting conditions do not cause complete loss of object visibility. TAIL-Safe starts with addressing the \textit{what} of \textit{task success assurance} through establishing three task-agnostic criteria to identify a policy's imminent failure. It then moves on to address the \textit{where} of \textit{task success assurance} through learning a Lipschitz-continuous Q-value function that maps the policy's state-action space to a safety score value leveraging the established criteria. The zero super level set of this function defines an \textit{empirical} safe set---approximating control invariance to the extent that the learned Q accurately reflects the policy's behavior. Finally, TAIL-Safe adapts Nagumo's tangentiality condition \cite{nagumo1942lage} to address the \textit{how} of \textit{task success assurance} by calculating a recovery action based on the gradient of the Q-value function; since we lack explicit dynamics, this is a heuristic adaptation where $\nabla_a Q$ serves as a proxy for the safety-increasing direction. We employed high-fidelity Gaussian Splatting techniques to learn the Q-value function. Experiments with a Franka Emika
robot demonstrate TAIL-Safe's ability to perform as a safety watchdog for flow-matching policies for two different real-world tasks.
\par The primary contributions of this paper are in:
\begin{itemize}
\item a principled approach to infer a learned policy's  ability to do a successful run-time execution (Section~\ref{sec:problem_formulation}). 
\item proposing and implementing three task-agnostic criteria for successful run-time execution of any manipulation task  (Section~\ref{sec:reward_function}). 
\item a framework to learn a Safety Q-value function that acts as a predictive landscape for task success assurance. 
\item a recovery filter that can guide a policy to regions with task success assurance (Section~\ref{sec:q_function} and Section~\ref{sec:recovery_controller}). 
    % \item \textbf{Real-to-Sim Pipeline for Failure Analysis:} We utilize Gaussian Splatting to reconstruct the environment, creating a photorealistic simulator that allows us to safely induce failures and explore states outside the expert demonstrations.
    % \item \textbf{Adaptive Safety Heuristic:} We introduce \textit{WeightNet}, which fuses modular scores—recognizability, visibility, and graspability—into a Lipschitz-continuous safety margin that adapts to the current task context.
    % \item \textbf{Safety Q-Value Landscape and Recovery:} We learn a Safety Q-value function that acts as a predictive landscape for task success. Using Nagumo's Theorem, we implement a runtime filter that uses the gradient of the Q-function to correct unsafe actions and keep the robot within the invariant region.
\end{itemize}
A key insight is that visibility, recognizability, and graspability exhibit natural Lipschitz structure: small end-effector displacements produce bounded changes in these scores. This physical smoothness implies that if $(s_1, a_1)$ and $(s_3, a_3)$ are both safe, an intermediate $(s_2, a_2)$ is likely safe as well. By enforcing Lipschitz continuity via spectral normalization, the learned Q-function enables reliable interpolation between observed safe and unsafe regions.

\section{Problem Formulation}
\label{sec:problem_formulation}
% We address the inherent fragility of deterministic imitation learning policies, which frequently fail even when operating strictly within their training distribution. Beyond standard distribution drift, these policies often exhibit extreme sensitivity to initial conditions; slight variations in the robot's starting configuration or accumulated approximation errors can drive the system into unstable regions of the state space. 
Let $\pi: \mathcal{S} \rightarrow \mathcal{A}$ be a deterministic visuomotor IL policy, e.g.  a policy learned through flow-matching algorithm~\cite{lipman2022flow}, that has learned a manipulation task from demonstrations $\mathcal{D} = \{(o_t, a_t, o_{t+1})\}_{t=1}^{T}$, where $o_t$ denotes the RGB-D observation and $a_t$ the corresponding expert action. %The policy is trained via a flow-matching algorithm~\cite{lipman2022flow} that learns to map observations to actions by matching the velocity field of expert trajectories. 
Our objective is to design the safety watchdog TAIL-Safe that enables $\pi$ to offer \textit{task success assurance}. We model TAIL-Safe as a Markov Decision Process (MDP) defined by the tuple $\mathcal{M}=(\mathcal{S},\mathcal{A},\mathcal{T},\mathcal{R},\gamma)$. The state space $\mathcal{S} \subseteq \mathbb{R}^{n}$ is the same as $\pi$'s state space and comprises the robot's proprioceptive data concatenated with a latent representation of the visual environment. The action space $\mathcal{A} \subseteq \mathbb{R}^{m}$ consists of continuous end-effector delta commands and is the same as  $\pi$'s action space. The transition dynamics $\mathcal{T}:\mathcal{S}\times\mathcal{A} \rightarrow \mathcal{S}$ represents the physical evolution  $s_{t+1}=\mathcal{T}(s_{t},a_{t})$ under the policy $\pi$ and is unknown due to run-time uncertainties -- such as, perceptual noises, environmental perturbations, changes in $\pi$'s starting position -- despite $\pi$ being deterministic. The reward function $\mathcal{R}$ captures the immediate benefit of taking an action at a state under the policy $\pi$, exclusively with respect to achieving the \textit{task success assurance} and is  unknown. The goal of TAIL-Safe is to learn an action-value function $Q(s,a): \mathcal{S} \times \mathcal{A}\rightarrow \mathbb{R}$ in such a way that, with an appropriately designed $\mathcal{R}$, it evaluates the long-term value of each state-action pair under policy $\pi$ with respect to achieving the \textit{task success assurance}. If we can design $\mathcal{R}$ such that $\mathcal{R}(s,a) \geq 0$ indicates satisfaction of all known task-critical criteria at $(s,a)$, and if we define $Q(s,a)$ to calculate the \textit{minimum} reward encountered along the trajectories, instead of the sum of rewards, $Q(s,a) \geq 0$ will indicate---based on the training data---that there exists a sequence of actions from $(s,a)$ that maintains all task-critical criteria until the end of trajectory. The zero-superlevel set $\{(s,a) : Q(s,a) \geq 0\}$ thus characterizes an empirical control invariant set from which task success is expected based on the training data. This transforms the Q-function into a binary safety indicator: positive values provide empirical assurance of task success, while negative values signal likely failure.  %Evaluating such a $Q(s,a)$ is computationally challenging due to the very high dimensional state space involved with visuomotor policy and the need for realistic rendering of the environments in $\mathcal{D}$. %The final objective of TAIL-Safe is to leverage the learned Q-function to design a recovery controller that can steer $\pi$ away from state-action pairs that may lead to task failure. 
\par Realizing TAIL-Safe requires solving the following challenges that are discussed in Section~\ref{sec:reward_function}--\ref{sec:recovery_controller}: 
\begin{itemize}
\item Defining a reward function $\mathcal{R}$ (Section~\ref{sec:reward_function}). %that maps a state-action pair $(s,a)$ to an instantaneous safety score with respect to task success in a task-agnostic manner  
\item Constructing the Q-function $Q(s,a)$ (Section~\ref{sec:q_function}).  %that can be leveraged to identify a control invariant set for $\pi$ 
\item Leveraging $Q(s,a)$ to design a recovery controller to prevent $\pi$ from entering into states that have no \textit{task success assurance} (Section~\ref{sec:recovery_controller}).
\end{itemize}

\section{Related Works}

 Prior approaches address safety through constraint-based training~\cite{liu2022constrained}, ensemble uncertainty~\cite{lakshminarayanan2017simple}, or human intervention~\cite{hoque2021lazydagger,kelly2019hgdagger}. Uncertainty-based methods (e.g., deep ensembles) can detect risky states but lack recovery mechanisms. Unlike these, TAIL-Safe learns an empirical safe set with gradient-based recovery without human supervision. Control Barrier Functions~\cite{ames2017cbf} and Hamilton-Jacobi reachability~\cite{bansal2017hamilton} provide formal guarantees but require explicit dynamics. Learning-based extensions~\cite{robey2020learning,dawson2022safe} partially address this but do not scale to high-dimensional visual observations. Latent Safety Filters~\cite{so2023solving} and LatentCBF~\cite{xiao2023barriernet} learn barrier functions in latent space with probabilistic certificates, often requiring dynamics models or latent dynamics learning. TAIL-Safe differs by avoiding dynamics modeling entirely, trading formal guarantees for reduced complexity and direct applicability to pre-trained visuomotor policies.

 Compared to model-based methods (CBFs, HJ reachability), TAIL-Safe trades formal guarantees for applicability to high-dimensional visuomotor policies without dynamics. Compared to uncertainty-based detection, we provide recovery rather than just flagging. Compared to learned latent filters, we avoid dynamics modeling by relying on trajectory-level outcomes.

\section{TAIL-Safe: A Safety Watchdog for Imitation Learning Policy}
\label{sec:tail_safe}
This section discusses the implementation of TAIL-Safe as formalized in Section II.
 \subsection{Task-agnostic Definition of the Reward Function $\mathcal{R}$}
\label{sec:reward_function}
\label{sec:safetycriteria}
The reward function $\mathcal{R}$ needs to capture, in short-term, how a state-action pair contributes in completing the task. The current version of TAIL-Safe operates with no privileged information such as demonstrations that show how the policy may fail~\cite{nakamura2025latent}. Instead, we propose three \textit{task-agnostic} criteria that needs to be fulfilled at a state in order for the policy to succeed in completing any manipulation task that involves grasping an object: 1) \textbf{Visibility} ($s_{fov}$): the object to be manipulated should remain within the sensor's field of view; 2) \textbf{Recognizability} ($s_{rec}$): the perception of the target object should have some degree of overlap with that of the training distribution; 3) \textbf{Graspability} ($s_{grasp}$): The geometric quality of potential contact with the target object should be high. We encode these criteria into a scalar safety heuristic function, $h: \mathcal{S} \times \mathcal{A} \rightarrow [-1, 1]$. \new{For non-grasping tasks (e.g., pushing, pouring), the graspability criterion can be replaced with task-appropriate geometric constraints (e.g., contact alignment for insertion).} The function $h$ can be modified to incorporate additional task-specific criteria without impacting the TAIL-Safe formulation in Section II. TAIL-Safe operationalizes these task-agnostic criteria and leaves extracting task-specific criteria as future work. The function $h$ is designed as a signed distance field for task success viability: a positive value $h(s, a) > 0$ indicates that all criteria are satisfied and task completion remains viable. Conversely, $h(s, a) < 0$ signifies that the system has failed to satisfy the requisite modular scores necessary for task completion. We use $h$ as the reward function.
\begin{equation}
    \mathcal{R}(s, a) \equiv h(s, a)
\end{equation}

\subsubsection{Evaluating Visibility Score, $s_{fov}$}
We project the target object's point cloud into camera coordinates and compute a geometric score based on point density and distance from the image center, yielding $s_{fov} \in [0,1]$ that decays as the object approaches peripheral regions (see Appendix~\ref{app:visibility}).

\subsubsection{Evaluating Recognizability Score, $s_{rec}$}
We extract 128-dimensional embeddings $\phi(o)$ from the policy's visual encoder, apply PCA (95\% variance), and fit a Gaussian $(\mu_{\mathcal{D}}, \Sigma_{\mathcal{D}})$ over expert demonstrations. The score is derived from the inverse Mahalanobis distance:
\begin{equation}
    s_{rec}(s) = \sigma\left(-d_M(\phi(o), \mu_{\mathcal{D}}, \Sigma_{\mathcal{D}})\right)
\end{equation}
where $\sigma(\cdot)$ maps to $[0,1]$, penalizing observations statistically distinct from the expert's experience. (see Appendix~\ref{app:visibility})

\subsubsection{Evaluating Graspability Score, $s_{grasp}$}
We sample antipodal grasp candidates from the segmented object point cloud~\cite{ten2017grasp} and score alignment with the current end-effector pose:
\begin{equation}
    s_{grasp}(s,a) = \max_{g \in \mathcal{G}} \; \text{sim}(T_{ee}(s,a), T_g)
\end{equation}
where $T_g \in SE(3)$ is the grasp transform,  $T_{ee}$ is the end-effector transform induced by action $a$ and $\text{sim}(\cdot)$ measures pose similarity via weighted translational and rotational distances. Object segmentation is performed using SAMv2~\cite{ravi2024sam2}, which provides robust masks even under partial occlusion. In our tabletop setting with simple, well-separated objects, segmentation errors were rare ($<$2\% of frames); when segmentation fails completely, $s_{grasp}$ defaults to zero, triggering conservative safety behavior.
(see Appendix~\ref{app:visibility})

\subsubsection{Learning the local score function $h$}
The goal of $h(s,a)$ (which is the same as $\mathcal{R}$) is to analyze $s_{fov}, s_{rec},$ and $s_{grasp}$ to assess the viability of completing the task from $(s,a)$. The challenge is that the contribution of these three components depends on the position of the state along the trajectory. %A fundamental challenge is that the definition of a ``safe state'', with respect to \textit{task-success assurance}, is dynamic and context-dependent. 
For instance, visual occlusion of the target should be avoided as the policy approaches the object but is unavoidable during grasp execution. Therefore, a state closer to the grasp location should not have low reward because of a poor $s_{fov}$ value (see Appendix~\ref{app:visibility}). We address this by learning a context-aware fusion of $s_{fov}, s_{rec},$ and $s_{grasp}$ via a multi-layer perceptron that we term \textbf{WeightNet}. Each of the three criteria got fused by WeightNet into a single scalar. Unlike a fixed weighted sum, WeightNet learns dynamic weights that adapt to the current task phase:
\begin{equation}
    h(s, a) = \sum_{i \in \{rec, fov, grasp\}} w_i(s) \cdot s_i(s, a) - \tau
\end{equation}
where $\tau$ is a learnable threshold such that positive values indicate safe states. The network outputs weights via softmax, enabling automatic reweighting---visibility dominates during approach, graspability near contact (Fig.~\ref{fig:weights}). We enforce Lipschitz continuity via $\ell_\infty$-norm row-wise weight normalization. WeightNet is trained via binary cross-entropy on trajectory-level labels; see Appendix~\ref{app:weightnet} for details.

\subsection{Constructing a $Q$ Function to Identify a Control Invariant Set for the Policy}
\label{sec:q_function}
Local safety checks incorporated in $\mathcal{R}$ are insufficient; a robot may be in a state with high $\mathcal{R}$ but may execute an action at any point of time in the future that can lead to a task failure. To capture this long-term behavior, we define a Safety Q-value function $Q(s, a)$. This function estimates the minimum safety margin encountered along the trajectory induced by the policy $\pi$:
\begin{equation}
    Q(s, a) = \mathbb{E}_{\pi} \left[ \min_{k \ge 0} \gamma^{k} \mathcal{R}(s_{t+k}, a_{t+k}) \mid s_t = s, a_t = a \right]
\end{equation}
Recall that $\mathcal{R}(s,a) \equiv h(s,a)$. This formulation frames the problem as a Reach-Avoid specification which is typically the way safety is formulated in control-theoretic safety literature \cite{fisac2015reach,hsu2023isaacs} (Details in Appendix~\ref{app:q_training}). A non-negative value, $Q(s, a) \ge 0$, serves as an empirical indicator that---based on the training data---the policy $\pi$ can complete the task as learned from demonstrations. Conversely, $Q(s, a) < 0$ acts as a predictive early warning, signaling that the current trajectory is likely driving the system toward a violation of the success criteria.
Therefore, by construction, the Control Invariant Set for the policy $\pi$, that we term $\mathcal{C}_{safe} \subset \mathcal{S} \times \mathcal{A}$, is the zero-superlevel set of $Q(s,a)$:
\begin{equation}
    \mathcal{C}_{safe} = \{(s, a) \in \mathcal{S} \times \mathcal{A} \mid Q(s, a) \ge 0\}
\end{equation}
If $Q$ were analytically tractable, then $C_{safe}$ would provide a formal guarantee of task success for all states it contains. In practice, the high dimensionality of the problem forces $Q$ to be estimated via function approximation, so $C_{safe}$ can offer only empirical assurance of task success. See the \textbf{Remark on Invariance} at the end of this subsection.

\subsubsection{Creating a Digital Twin from $\mathcal{D}$ to Learn $Q$}
A fundamental challenge in learning the boundary of $\mathcal{C}_{safe}$ is the inherent requirement for failure data. To accurately characterize where $\pi$ may fail, TAIL-Safe must observe the policy in states that lead to task failure. However, collecting such data on a physical robot is both dangerous and operationally inefficient -- the very failures we need to observe are precisely the ones we wish to prevent. To address this, we construct a high-fidelity digital twin using Gaussian Splatting~\cite{kerbl20233d} from 100 RGB images ($\sim$20 min total: 5 min capture + 15 min reconstruction). We employ SAMv2~\cite{ravi2024sam2} for segmentation and DINOv2~\cite{oquab2023dinov2} features for accurate object boundaries. The twin renders at $>$30 FPS while preserving visual properties $\pi$ was trained on, with an integrated Franka kinematic model enabling policy rollouts matching physical observations. \new{Crucially, the twin is \emph{not} a frozen pre-built model used at deployment: each tracked object's Gaussians are updated online from the wrist-mounted RGB-D stream so that the renders seen by the policy mirror the current real-world state. The twin therefore acts as a live visual mirror; physical contact, friction, and dynamics are handled by the real robot, and the Lipschitz continuity of $Q$ (Sec.~IV-B) bounds the effect of small rendering artifacts on safety predictions (Appendix~\ref{app:twin_live}).} The digital twin enables systematic exploration beyond $\mathcal{D}$. We perturb $\pi$ in three ways: object displacement ($\pm 5$cm) and rotation ($\pm 30^\circ$), Gaussian noise injection ($\sigma \in [0.01, 0.05]$), and gradient-based perturbations toward the safe set boundary. This yields $\sim$500 rollouts per task ($\sim$40\% failures), recording state, action, safety scores, and outcomes at each timestep. This dataset---successful trajectories defining $\mathcal{C}_{safe}$'s interior and failures characterizing its boundary---forms the basis for training Q.

\subsubsection{Reach-Avoid Bellman operator}
Standard RL formulates Q as cumulative discounted rewards, but this is ill-suited for \textit{task success assurance}---a single safety violation invalidates the entire trajectory. We instead frame the problem as \textit{reach-avoid}: $\pi$ must reach task completion while strictly avoiding any $(s,a)$ where $h(s,a) < 0$. We replace the standard Bellman recursion with one propagating the \textit{minimum}:
\begin{equation}
    Q^\pi(s_t, a_t) = \min\left(h(s_t, a_t),\; \gamma \cdot Q^\pi(s_{t+1}, \pi(s_{t+1}))\right)
\end{equation}
The $\min$ operator ensures that $Q(s,a)$ is upper-bounded by the lowest safety score encountered anywhere along the future trajectory. Therefore, $Q(s,a) \geq 0$ indicates---based on the observed training trajectories---that a sequence of actions maintaining all task criteria until task completion is likely achievable, while $Q(s,a) < 0$ signals likely failure. Appendix~\ref{app:q_training} provides additional details on training the Q-function.
%This defines \textit{what} Q should predict; 

We next address \textit{how} of \textit{task-success assurance} by  ensuring that Q provides informative gradients for generating corrective action.
\begin{figure}[t]
    \centering
    \includegraphics[width=0.90\columnwidth]{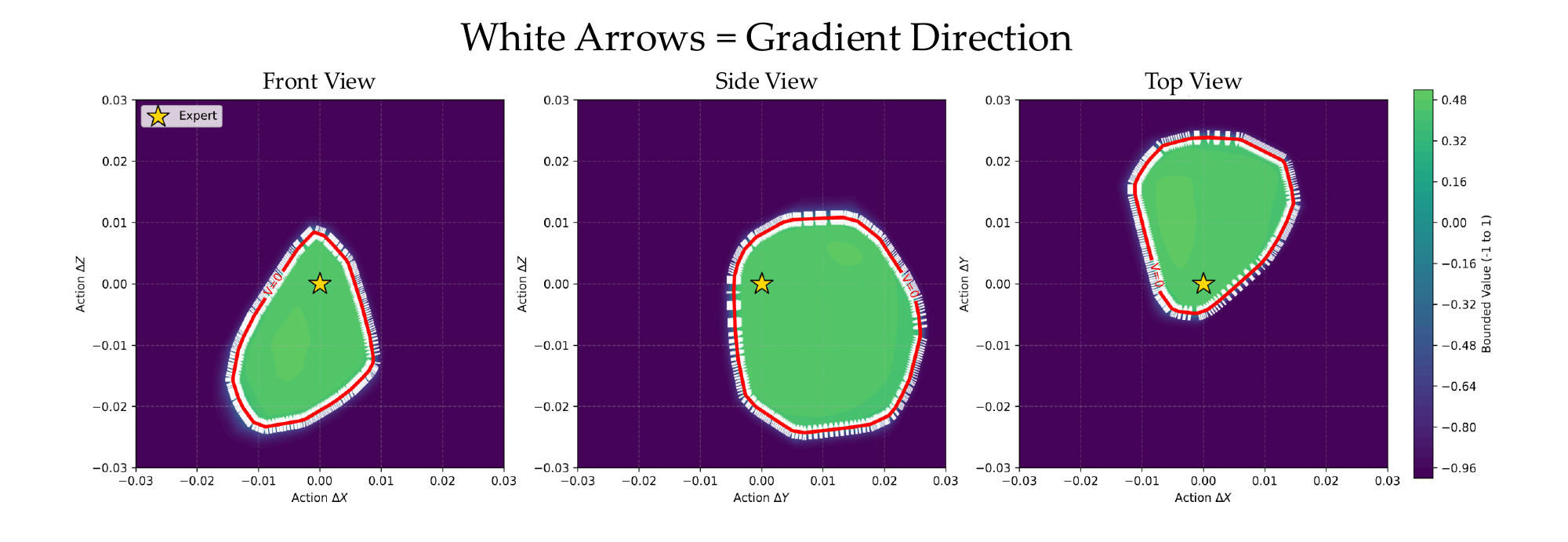}
    \caption{\textbf{Q-Function Landscape.} Multi-view 2D projections show the bounded safe set (green, $Q \geq 0$) in action space. White arrows indicate $\nabla_a Q$ pointing inward, enabling gradient-based recovery.}
    \label{fig:safeset_comparison}
\end{figure}

\begin{figure}[t]
    \centering
    \includegraphics[width=0.80\columnwidth]{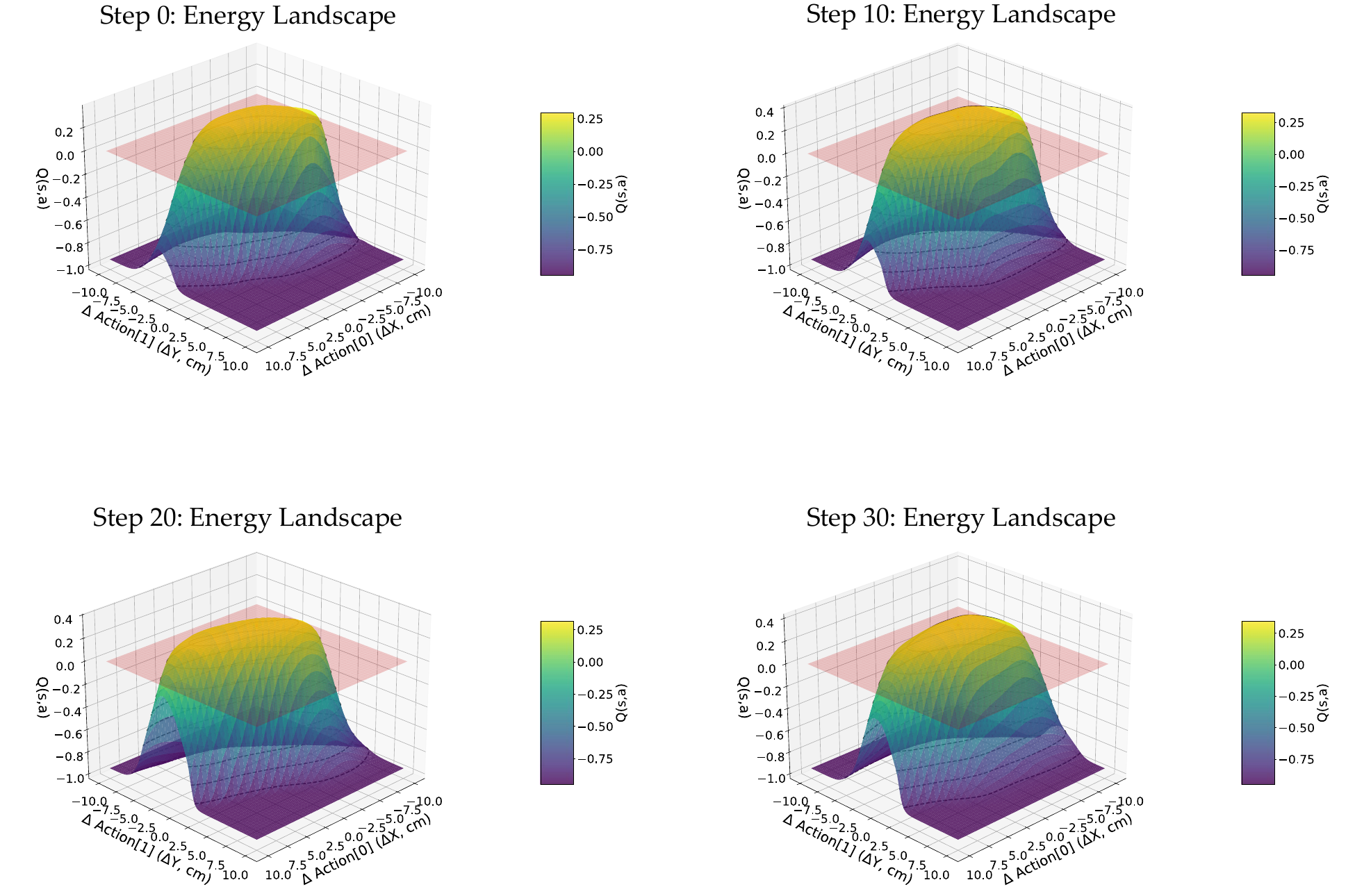}
    \caption{\textbf{Q-Function as Bounded Hill.} The Q-function peaks at expert actions and decays smoothly, ensuring $\nabla_a Q$ points toward safe configurations for gradient-based recovery.}
    \label{fig:q_landscape}
\end{figure}

\subsubsection{Learning Q as an energy function}
Regressing Q to reach-avoid targets provides correct predictions but insufficient gradients for recovery. In regions far from $\mathcal{D}$, standard regression produces flat or noisy gradients---precisely where correction is needed most. To address this, we shape $Q(s, a)$ as an energy function with a well-defined landscape (Figure~\ref{fig:q_landscape}). State-action pairs from successful trajectories represent local safety optima. We impose a \textit{Bounded Hill} geometry: Q peaks at demonstrated $(s, a)$ pairs and decays smoothly as actions deviate. This is enforced via hinge loss:
\begin{equation}
    \mathcal{L}_{hill} = \mathbb{E}_{(s,a^*) \sim \mathcal{D}} \left[ \max(0, Q(s, a) - (1 - \alpha \|a - a^*\|_2)) \right]
\end{equation}
where $a^*$ is the action from a successful trajectory at state $s$, and $\alpha > 0$ controls decay rate. During training, we sample $a$ by adding uniform noise $\mathcal{U}(-0.1, 0.1)$ to $a^*$, ensuring the hill is shaped in a neighborhood around demonstrated actions rather than globally. This prevents distorted level sets far from data. The complete objective combines reach-avoid regression with this constraint:
\begin{equation}
    \mathcal{L}_Q = \mathcal{L}_{anchor} + \lambda_{hill}\mathcal{L}_{hill}
\end{equation}
where $\lambda_{hill} > 0$ balances reach-avoid accuracy against gradient quality (we use $\lambda_{hill} = 0.1$). Larger values produce stronger gradients but may distort level sets; smaller values preserve fidelity but slow recovery. \new{Because $\mathcal{L}_{hill}$ acts only on the gradient landscape around demonstrated actions and the reach-avoid loss $\mathcal{L}_{anchor}$ remains dominant, the decision boundary $\{Q\!=\!0\}$ is preserved (validated by 99.3\% AUROC, Table~\ref{tab:calibration}).} We apply hill shaping only at states from successful trajectories to prevent artificial safe regions (see Appendix~\ref{app:energy} for details). This achieves 100\% recovery success with 99.3\% state-level accuracy (Table~\ref{tab:weightnet_ablation}).

This formulation ensures that $\nabla_{a}Q$ is non-vanishing and points toward safe configurations, even in states not encountered during training (Figure~\ref{fig:safeset_comparison}). To ensure these gradients remain stable under small perturbations, we additionally enforce Lipschitz continuity.

\subsubsection{Lipschitz continuity of Q}
Since we will leverage Q's gradient to design a recovery controller, if Q exhibits sharp discontinuities, the gradient $\nabla_a Q$ can change erratically with small changes in the action, causing oscillatory or unpredictable corrections to the robot's end-effector commands. This makes it difficult to ensure reliable convergence to a safe action. We enforce smoothness by implementing the Q-function as a 4-layer MLP with spectral normalization~\cite{miyato2018spectral} on each layer. Each weight matrix $W$ is normalized by its largest singular value $\sigma(W)$, estimated via power iteration. We use Softplus activations ($\beta=5.0$) instead of ReLU for gradient continuity. This bounds the global Lipschitz constant $L_Q \leq 2.5$:
\begin{equation}
    \|\nabla_{a}Q_{\phi}(s, a)\|_2 \leq L_Q = 2.5 \quad \forall (s, a)
\end{equation}
We validated this bound over 10,000 samples (max: 2.31, mean: 0.87); see Appendix~\ref{app:lipschitz}. This constraint bounds the recovery process (\textbf{Proposition 1}). 

\textbf{Remark on Invariance.} We use the term ``empirical safe set'' rather than claiming formal control invariance. If $Q(s,a) \geq 0$, the training data indicates the policy can maintain safety---but this is an \textit{empirical} guarantee contingent on Q's accuracy and the runtime environment matching the training distribution. The gradient-based recovery uses $\nabla_a Q$ as a \textit{heuristic proxy} for safety-increasing directions; without explicit dynamics, we cannot enforce true tangentiality conditions.

\begin{figure*}[!t]
\centering
\includegraphics[width=0.70\textwidth]{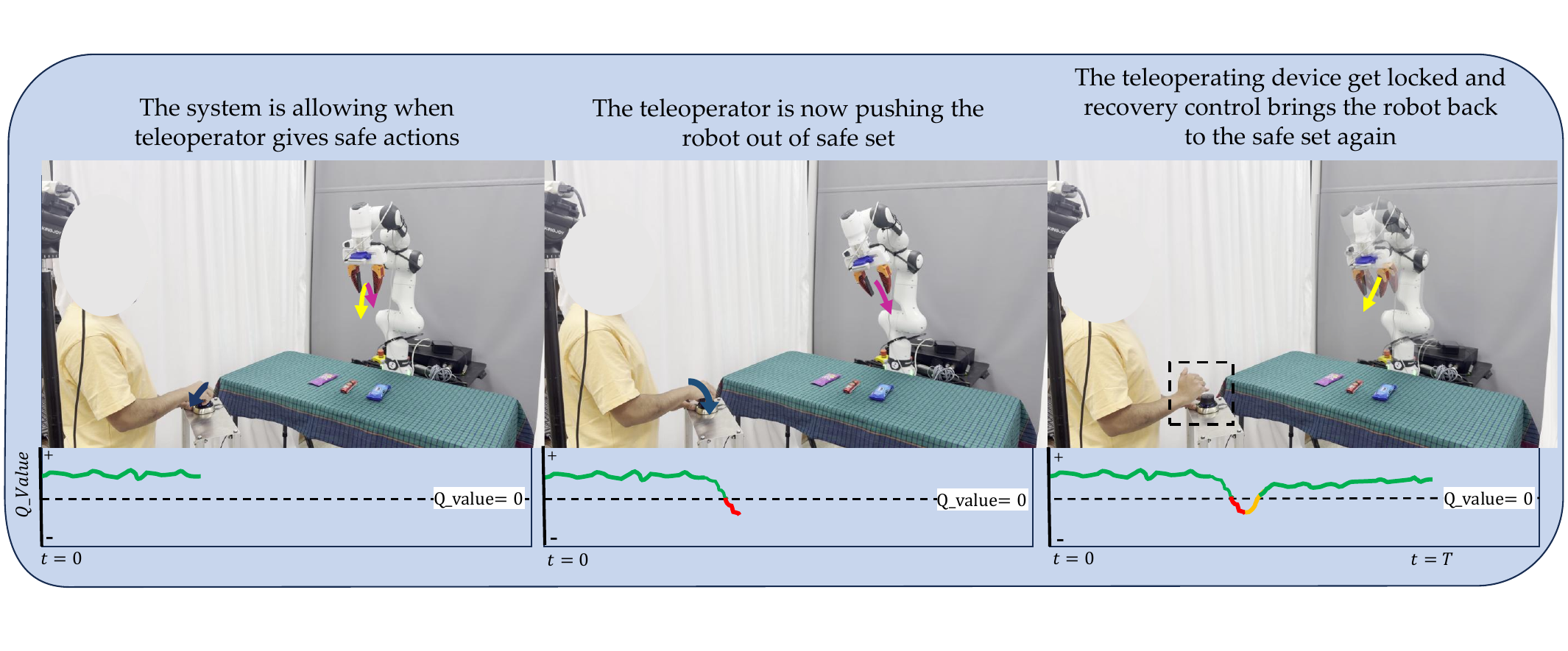}
\caption{\textbf{TAIL-Safe Recovery During Teleoperation.} \textbf{(Left)} Safe actions with $Q > 0$. \textbf{(Middle)} Teleoperator pushes robot out of safe set. \textbf{(Right)} TAIL-Safe locks device and activates recovery, steering back to safety. Bottom plots show Q-value trajectory recovering to positive values.}
\label{fig:wide_figure}
\end{figure*}

\subsection{Designing a Recovery Controller to Maintain Task Success}
\label{sec:recovery_controller}
TAIL-Safe monitors all proposed actions under $\pi$ to ensure task-success assurance during deployment (as illustrated in Figure~\ref{fig:teaser}). At any time, %the filter evaluates $Q(s_t, a_\pi)$ for the proposed action $a_\pi = \pi_\theta(s_t)$. 
if $Q(s, a) \geq 0$, the action lies within $\mathcal{C}_{safe}$ and executes unmodified; if $Q(s, a) < 0$, the filter initiates a gradient-based recovery. The proposed recovery is inspired by Nagumo's tangentiality condition \cite{nagumo1942lage,blanchini2008set}, which states that a set is forward-invariant if boundary controls point inward. Since we lack explicit dynamics, we use $\nabla_a Q$ as a proxy for the safety-increasing direction via projected gradient ascent:
\begin{equation}
    a^{(k+1)} = \text{Proj}_{\mathcal{A}} \left( a^{(k)} + \eta \frac{\nabla_{a} Q(s_t, a^{(k)})}{\|\nabla_{a} Q(s_t, a^{(k)})\|_2} \right)
\end{equation}
where $\text{Proj}_{\mathcal{A}}$ clips to kinematic limits. Starting from $a^{(0)} = a_\pi$, we iterate until $Q \geq 0$ or halt after $k_{max} = 10$ iterations. We set $\eta = 0.05$ based on grid search; convergence typically requires 3--5 iterations (mean: 2.3), enabling 20Hz operation. See Appendix~\ref{app:hyperparams} for hyperparameter sensitivity.

\begin{proposition}[Bounded Recovery Step]
\label{prop:bounded_recovery}
Under Lipschitz constraint $L_Q$, the recovery update $\Delta a = \eta \nabla_a Q / \|\nabla_a Q\|_2$ satisfies $\|\Delta a\|_2 = \eta$. The safety improvement per step is lower-bounded: $Q(s, a + \Delta a) - Q(s, a) \geq \eta \cdot c$ for some $c > 0$ when $\nabla_a Q \neq 0$.
\end{proposition}

\noindent The proof is in Appendix~\ref{app:proof}. Proposition~\ref{prop:bounded_recovery} ensures predictable step magnitude $\eta$ and expected progress toward $\mathcal{C}_{safe}$ when the gradient is non-zero.
\section{Experimental Evaluation}
\label{sec:experiments}
We evaluate TAIL-Safe on a Franka Emika robot in simulation and real-world settings with object displacement ($\pm 5$cm), rotation ($\pm 30^\circ$), and off-nominal starting configurations. Our experiments address: (1) Does TAIL-Safe enable consistent success? (Section~\ref{sec:exp_success_rate}) (2) What perturbations can it handle? (Section~\ref{sec:exp_resilient}) (3) How do components ($s_{fov}, s_{rec},$ and $s_{grasp}$) contribute? (Section~\ref{sec:exp_ablation})

\subsection{Tasks and Experimental Configuration}
\label{sec:exp_setup}

\begin{figure}[t]
    \centering
    \includegraphics[width=0.42\textwidth]{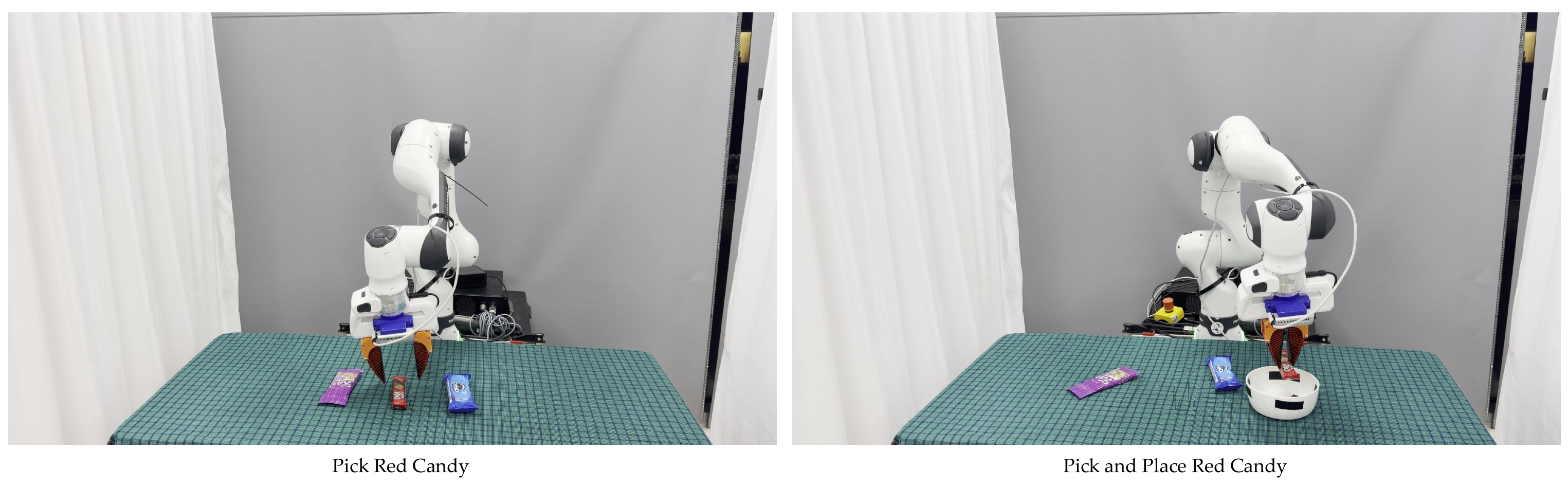}
    \caption{\textbf{Experimental Tasks.} Real robot setup for (left) Candy Picking and (right) Pick-and-Place under object perturbations ($\pm 5$cm, $\pm 30^\circ$).}
    \label{fig:tasks}
\end{figure}

% \begin{figure}[t]
%     \centering
%     % Top Row
%     \begin{subfigure}{\columnwidth}
%         \centering
%         \includegraphics[width=\linewidth]{images/pick_red_candy_task-compressed.pdf}
%         % \caption{\textbf{Pick Red Candy} }
%         \label{fig:failures_pick}
%     \end{subfigure}
    
%     \vspace{-1em}

%     % Bottom Row
%     \begin{subfigure}{\columnwidth}
%         \centering
%         \includegraphics[width=\linewidth]{images/pick_and_place_red_candy_task-compressed.pdf}
%         % \caption{\textbf{(Pick-and-Place)}}
%         \label{fig:failures_place}
%     \end{subfigure}
%     \vspace{-2.5em}
%     \caption{\textbf{Types of Tasks}}
%     \label{fig:failures}
% \end{figure}

We evaluate TAIL-Safe on two tabletop manipulation tasks (Figure~\ref{fig:tasks}).

\textbf{Candy Picking.} The robot grasps a specific candy from a cluttered workspace with distractors. We perturb objects ($\pm 5$cm displacement, $\pm 30^\circ$ rotation) from demonstration poses.

\textbf{Pick-and-Place.} The robot transfers an object to a designated bowl. The extended horizon makes this task more susceptible to compounding drift.

For each task, we collect 45 kinesthetic demonstrations and generate $\sim$500 perturbed rollouts in the Gaussian Splatting simulator. The robot operates at 20Hz with RGB-D from a wrist-mounted camera (43ms latency on RTX 4090). State space: 149D (128 visual + 21 proprioceptive); action space: 8D (7-DoF delta + gripper). WeightNet: 3-layer Lipschitz MLP (256 units); Q-function: 4-layer Lipschitz MLP (512 units) with spectral normalization ($L_Q \leq 2.5$) and Softplus activations. Q-labels computed via reach-avoid hindsight: $y_t = \min_{k \geq t} h(s_k, a_k)$.

\subsection{Policy Failure vs. TAIL-Safe Guided Success}
\label{sec:exp_success_rate}
Figure~\ref{fig:failures} illustrates failure modes in the vanilla policy: translation errors, incorrect orientations, wrong object selection, and inability to reach goals. Experiments demonstrate that flow-matching policies achieve consistent success when guided by TAIL-Safe.  We evaluate 200 rollouts per condition in simulation and 50 on the real robot.

% TODO: Add figure showing side-by-side comparison of failed vs successful rollouts
% Figure placeholder for qualitative visualization

\begin{figure}[t]
    \centering
    % Top Row
    \begin{subfigure}{0.72\columnwidth}
        \centering
        \includegraphics[width=\linewidth]{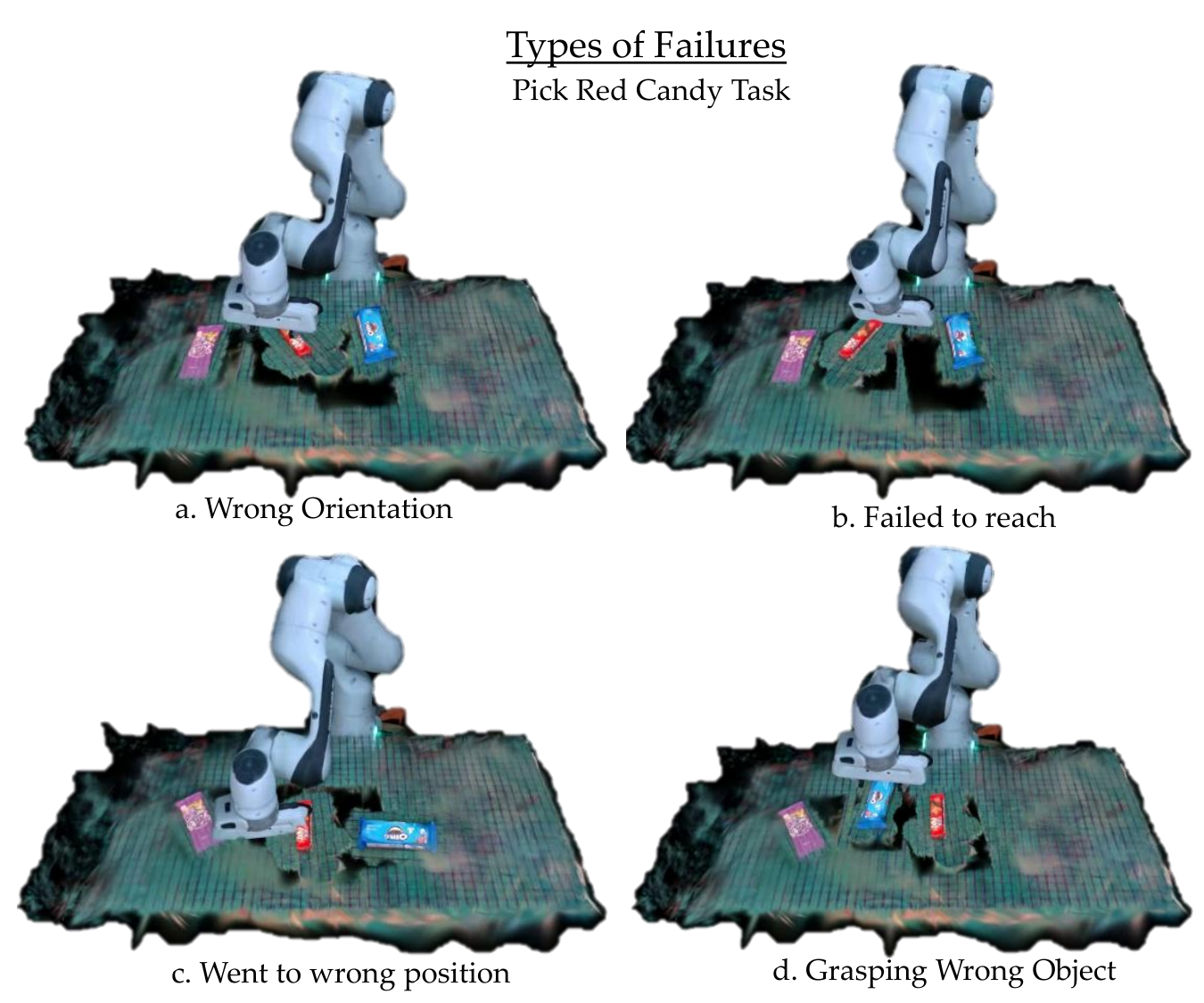}
        \caption{\textbf{Pick Red Candy:} (a) Wrong grasp orientation. (b) Fails to reach goal. (c) Translation error causes grasp failure. (d) Wrong candy selected.}
        \label{fig:failures_pick}
    \end{subfigure}
    
    \vspace{0.3em}

    % Bottom Row
    \begin{subfigure}{0.72\columnwidth}
        \centering
        \includegraphics[width=\linewidth]{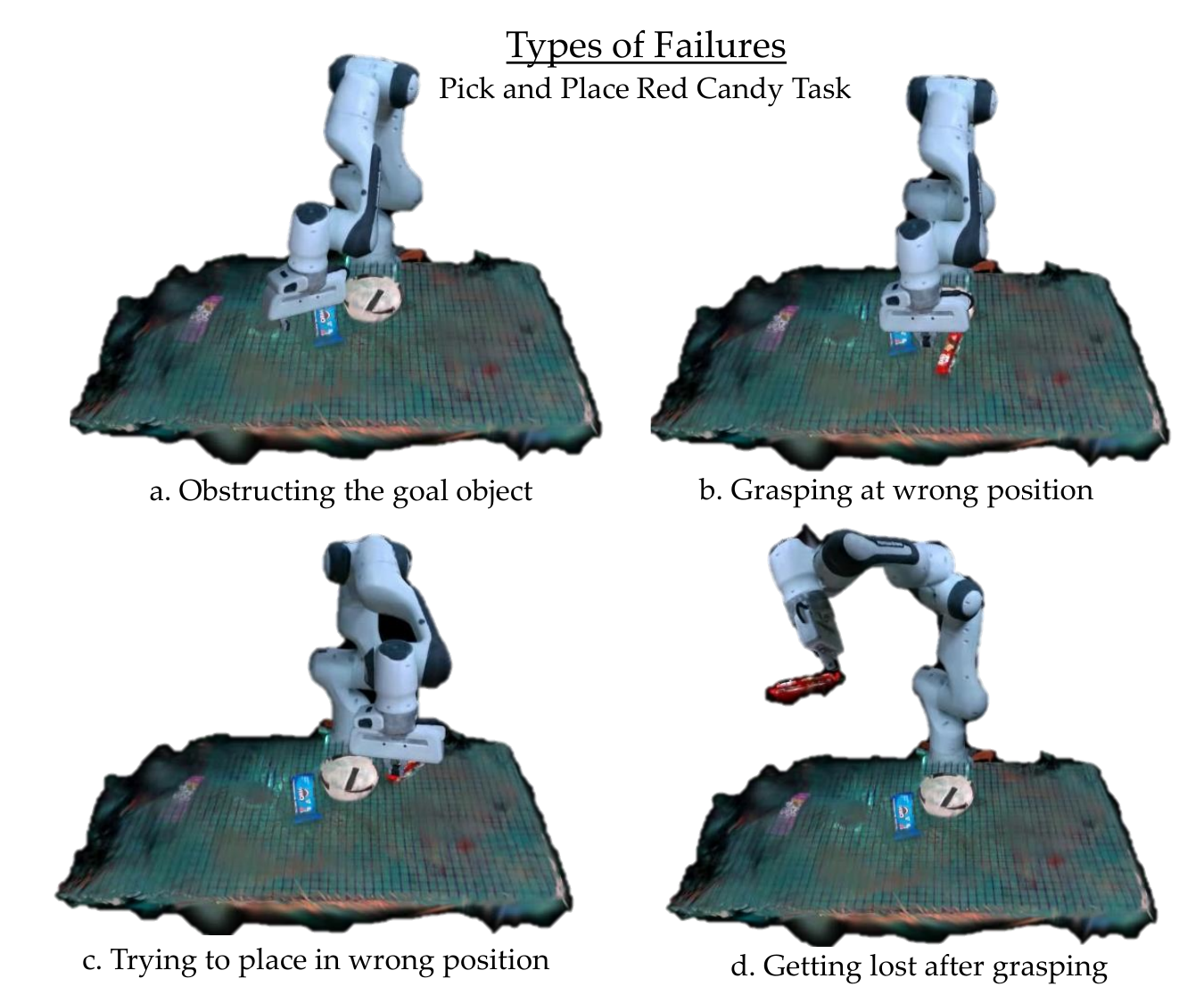}
        \caption{\textbf{Pick-and-Place:} (a) Gripper occludes target. (b) Wrong grasp position. (c) Placement failure. (d) Gets lost after picking.}
        \label{fig:failures_place}
    \end{subfigure}

    \caption{\textbf{Failure Modes in Vanilla Flow-Matching Policy.}}
    \label{fig:failures}
\end{figure}

% \begin{figure}[t]
%     \centering
%     \includegraphics[width=0.50\textwidth]{images/failures.pdf}
%     \caption{\textbf{Types of Failures in Vanilla Flow-Matching Policy.} \textbf{Top row (Pick Red Candy):} (a) Translation error propagates, causing grasp failure. (b) Policy fails to predict correct grasp orientation when candy is rotated. (c) Policy selects wrong candy to pick. (d) Policy fails to reach the goal position. \textbf{Bottom row (Pick-and-Place):} The task inherits all failure modes from picking, plus additional challenges. (e) Gripper occludes the target candy, causing the policy to lose track of the object. (f) After grasping, policy moves toward wrong direction instead of the bowl. (g) Policy fails during the final placement phase. (h) Policy successfully picks the candy but gets lost, moving to an unexpected pose and failing to locate the bowl.}
%     \label{fig:failures}
% \end{figure}

\begin{table}[t]
    \centering
    \caption{\textbf{Task Performance: Simulation and Real Robot}}
    \label{tab:performance_results}
    \footnotesize
    \setlength{\tabcolsep}{3pt}
    \renewcommand{\arraystretch}{1.2}
    \begin{tabular}{l l c c c c}
        \toprule
        \textbf{Task} & \textbf{Condition} & \textbf{N} & \textbf{Success} & \textbf{Steps} & \textbf{Recov.} \\
        \midrule
        \multicolumn{6}{c}{\textit{Simulation}} \\
        \midrule
        \multirow{2}{*}{Candy Pick} & Unsafe (No Safety) & 200 & 20.0\% & 94.2 $\pm$ 7.8 & --- \\
        & Safe (With Safety) & 200 & \textbf{100\%} & \textbf{74.3 $\pm$ 5.9} & 0.63 $\pm$ 0.7 \\
        \midrule
        \multirow{2}{*}{Pick-Place} & Unsafe (No Safety) & 200 & 23.3\% & 92.8 $\pm$ 8.6 & --- \\
        & Safe (With Safety) & 200 & \textbf{100\%} & \textbf{76.9 $\pm$ 6.8} & 0.70 $\pm$ 0.9 \\
        \midrule
        \multicolumn{6}{c}{\textit{Real Robot}} \\
        \midrule
        \multirow{2}{*}{Candy Pick} & Unsafe (No Safety) & 50 & 25.0\% & 91.5 $\pm$ 9.1 & --- \\
        & Safe (With Safety) & 50 & \textbf{100\%} & \textbf{78.2 $\pm$ 7.3} & 0.80 $\pm$ 0.8 \\
        \midrule
        \multirow{2}{*}{Pick-Place} & Unsafe (No Safety) & 50 & 20.0\% & 95.3 $\pm$ 8.4 & --- \\
        & Safe (With Safety) & 50 & \textbf{100\%} & \textbf{81.4 $\pm$ 7.9} & 0.85 $\pm$ 1.0 \\
        \bottomrule
    \end{tabular}
\end{table}
% \begin{figure}[t]
%     \centering
%     \includegraphics[width=0.50\textwidth]{images/recovery_roll.pdf}
%     \caption{\textbf{Gradient-Based Recovery.} When $Q(s,a)$ becomes negative, the runtime filter applies gradient-based corrections that project the action back toward the safe set.}
%     \label{fig:recovery_roll}
% \end{figure}
% \begin{table}[t]
%     \centering
%     \caption{\textbf{Safety Q-Function Evolution}}
%     \label{tab:safety_stats}
%     \footnotesize
%     \setlength{\tabcolsep}{3.5pt}
%     \renewcommand{\arraystretch}{1.3}
%     \begin{tabular}{l c c c}
%         \toprule
%         & \multicolumn{3}{c}{\textbf{Safety Value Statistics} (Mean $\pm$ Std)} \\
%         \cmidrule(lr){2-4}
%         \textbf{Condition} & \textbf{$Q_{start}$} & \textbf{$Q_{min}$} & \textbf{$Q_{final}$} \\
%         \midrule
%         Without TAIL-Safe & $-0.44 \pm 0.01$ & $-1.0 \pm 0.00$ & $-1.0 \pm 0.00$ \\
%         With TAIL-Safe & $+0.27 \pm 0.02$ & $-0.19 \pm 0.13$ & $+0.92 \pm 0.02$ \\
%         \bottomrule
%     \end{tabular}
% \end{table}

\begin{table}[h]
\centering
\caption{\textbf{WeightNet vs Equal Weights for Safety Detection.} WeightNet achieves near-perfect detection at both trajectory and state levels, while Equal Weights fails to identify safe regions (0\% trajectory, 1.9\% state recall). See Table~\ref{tab:weightnet_ablation} for full details.}
\label{tab:weightnet_detection}
\resizebox{\linewidth}{!}{
\begin{tabular}{@{}lcccc@{}}
\toprule
& \multicolumn{2}{c}{\textbf{Trajectory-Level}} & \multicolumn{2}{c}{\textbf{State-Level}} \\
\cmidrule(lr){2-3} \cmidrule(lr){4-5}
\textbf{Method} & \textbf{Safe Recall} & \textbf{Unsafe Recall} & \textbf{Safe Recall} & \textbf{Unsafe Recall} \\
\midrule
\textbf{WeightNet} & \textbf{100.0\%} & 98.3\% & \textbf{98.4\%} & 100.0\% \\
Equal Weights & 0.0\% & 100.0\% & 1.9\% & 100.0\% \\
\bottomrule
\end{tabular}
}
\end{table}

Table~\ref{tab:performance_results} summarizes results. The vanilla policy achieves 20--25\% success; with TAIL-Safe, it achieves 100\% on both tasks with more efficient episodes (74--81 vs. 92--95 steps) and $<$1 recovery per episode. We deploy directly on the physical robot without fine-tuning, achieving consistent 100\% success. Real-world requires slightly more recoveries (0.80--0.85 vs. 0.63--0.70) due to sensor noise. \textbf{Inference latency:} The reported 2.8ms covers Q-forward pass (1.2ms), gradient computation (0.9ms), and recovery iteration (0.7ms). Perception (SAMv2, DINOv2) runs asynchronously at 10Hz; visibility, graspability and recognizability scores are cached and updated per perception cycle, not per control step.

\subsection{Determining the Resilience of  TAIL-Safe to Different Type of Runtime Perturbations}
\label{sec:exp_resilient}

We qualitatively demonstrate TAIL-Safe's resilience to runtime perturbations. Figure~\ref{fig:sim_roll} shows the key difference: without TAIL-Safe, Q remains negative until failure; with TAIL-Safe, recovery activates when $Q < 0$ and steers back to safety.

% TODO: Add figure showing examples of perturbations that TAIL-Safe handles successfully
% Categories to visualize:
% - Object displacement (within ±5cm)
% - Object rotation (within ±30°)
% - Partial visual occlusion
% - Sensor noise
% - Initial position variations

\begin{figure}[t]
    \centering
    \begin{subfigure}[b]{0.40\columnwidth}
        \centering
        \includegraphics[width=\textwidth]{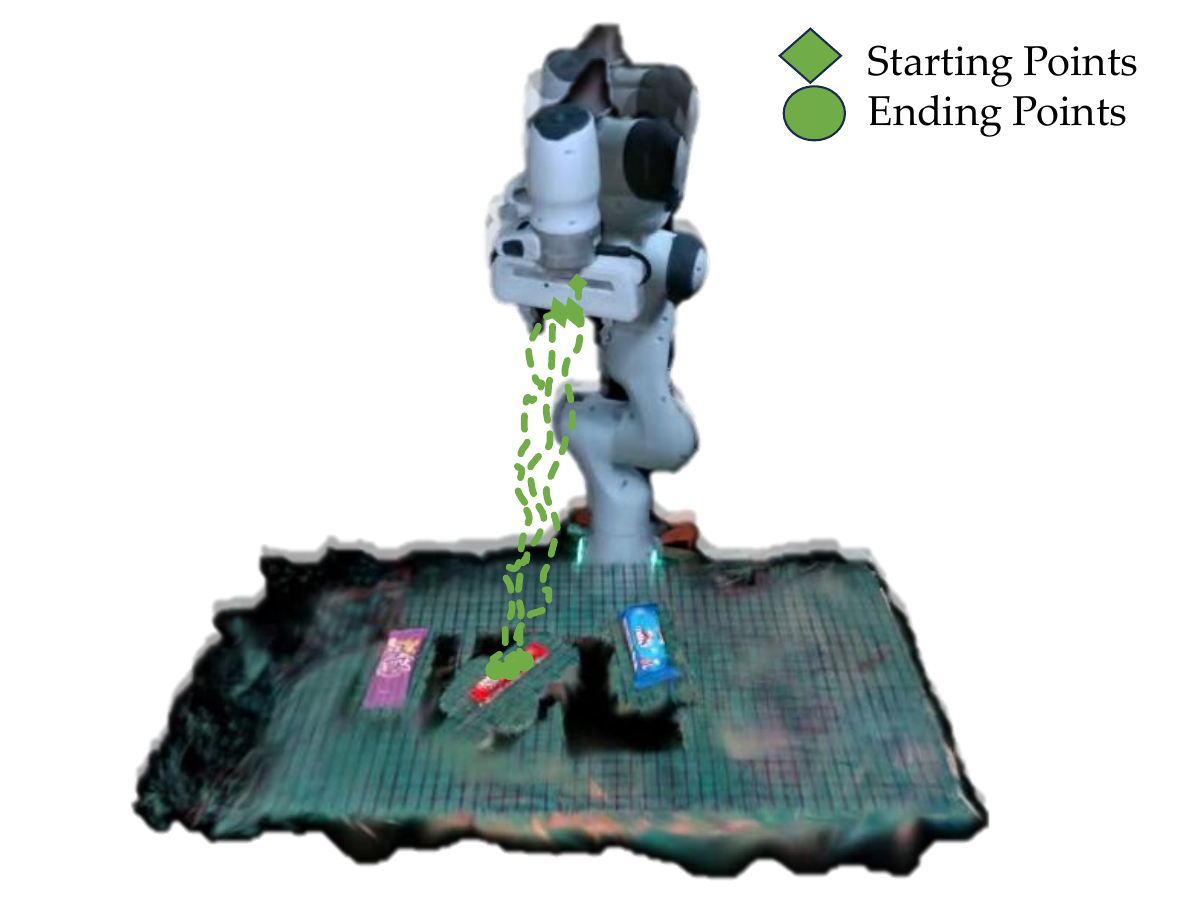}
        \caption{Safe Rollouts}
        \label{fig:safe_rollouts}
    \end{subfigure}
    \hfill
    \begin{subfigure}[b]{0.40\columnwidth}
        \centering
        \includegraphics[width=\textwidth]{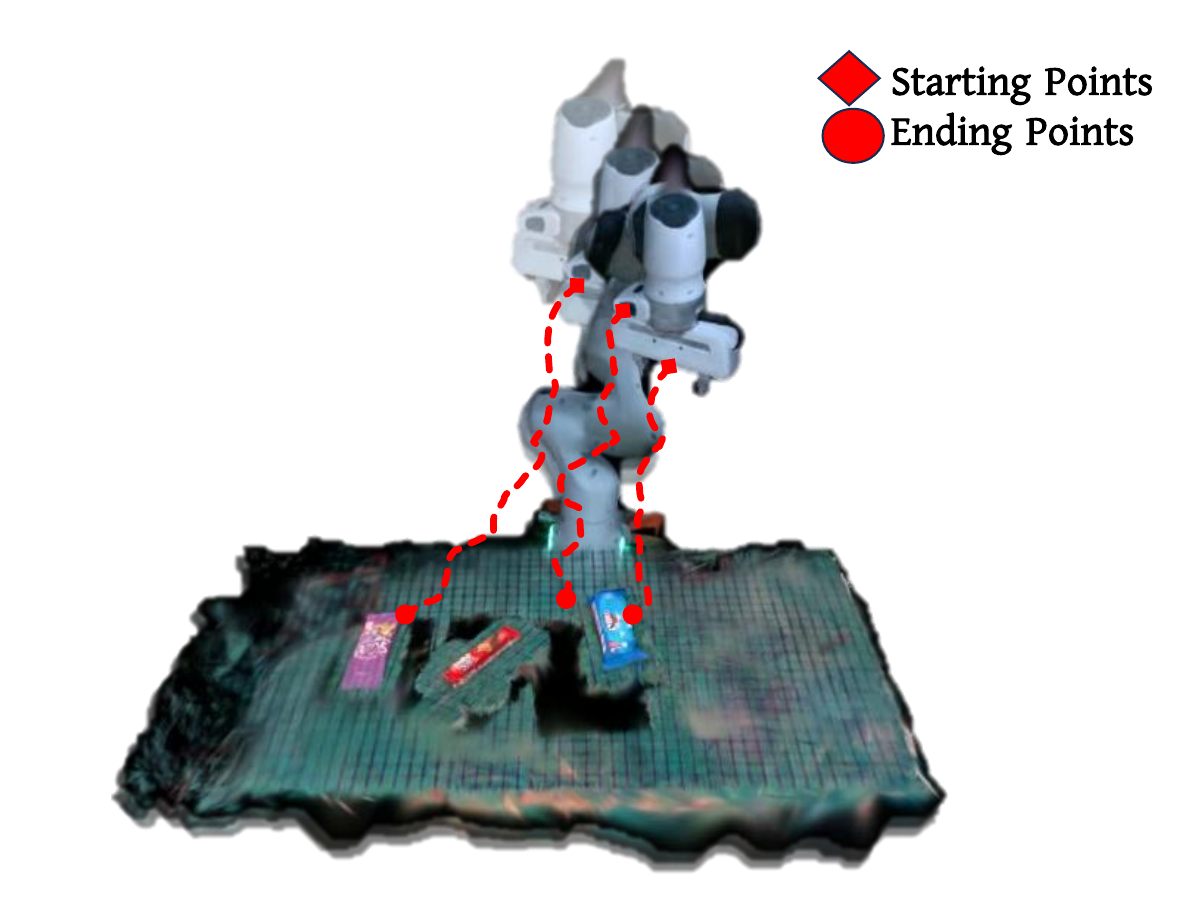}
        \caption{Unsafe Rollouts}
        \label{fig:unsafe_rollouts}
    \end{subfigure}
    \caption{\textbf{Trajectory Distribution.} \textbf{(a)} Safe initializations produce clustered paths reaching the target. \textbf{(b)} Unsafe initializations produce scattered, erratic paths.}
    \label{fig:trajectory_distribution}
\end{figure}

\begin{figure}[t]
    \centering
    \includegraphics[width=0.49\textwidth]{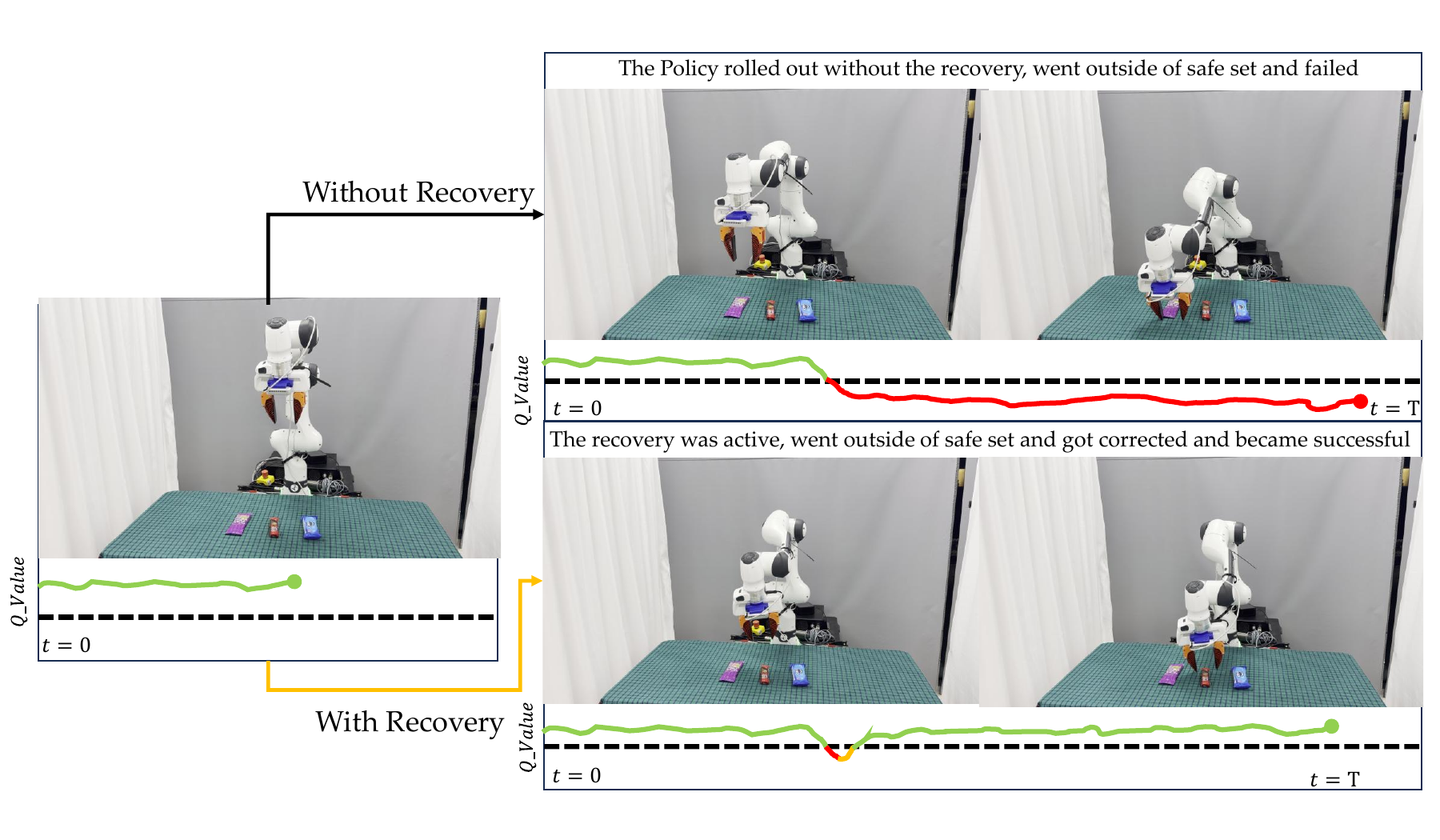}
    \caption{\textbf{Safe Set and Q-Value Propagation.} \textbf{(a)} Without TAIL-Safe: Q-value remains negative, task fails. \textbf{(b)} With TAIL-Safe: recovery controller steers robot back to safety upon detecting $Q < 0$.}
    \label{fig:sim_roll}
\end{figure}

% \begin{figure}[t]
%     \centering
%     \includegraphics[width=0.50\textwidth]{images/recovery_roll.pdf}
%     \caption{\textbf{Safe Set and Q-Value Propagation.} Both trajectories start from the same initial position outside the safe set. \textbf{(a) Without TAIL-Safe:} The robot begins outside the safe set (yellow ellipse) and the Q-value (red curve) remains negative throughout execution. Without recovery, the robot stays outside the closest safe set at every timestep, and the task eventually fails. \textbf{(b) With TAIL-Safe:} Starting from the same unsafe initial position, the recovery controller immediately activates upon detecting $Q(s,a) < 0$. Using the gradient of the Q-value function, the controller steers the robot back into the safe set. Once inside, the Q-value becomes positive (green), the robot remains within the safe set for the remainder of execution, and the task completes successfully.}
%     \label{fig:sim_roll}
% \end{figure}

% Placeholder for detailed qualitative analysis of handled perturbations

%\subsection{Runtime Perturbations TAIL-Safe Cannot Handle}
%\label{sec:exp_limitations}

\begin{figure*}[!t]
\centering
\includegraphics[width=0.85\textwidth]{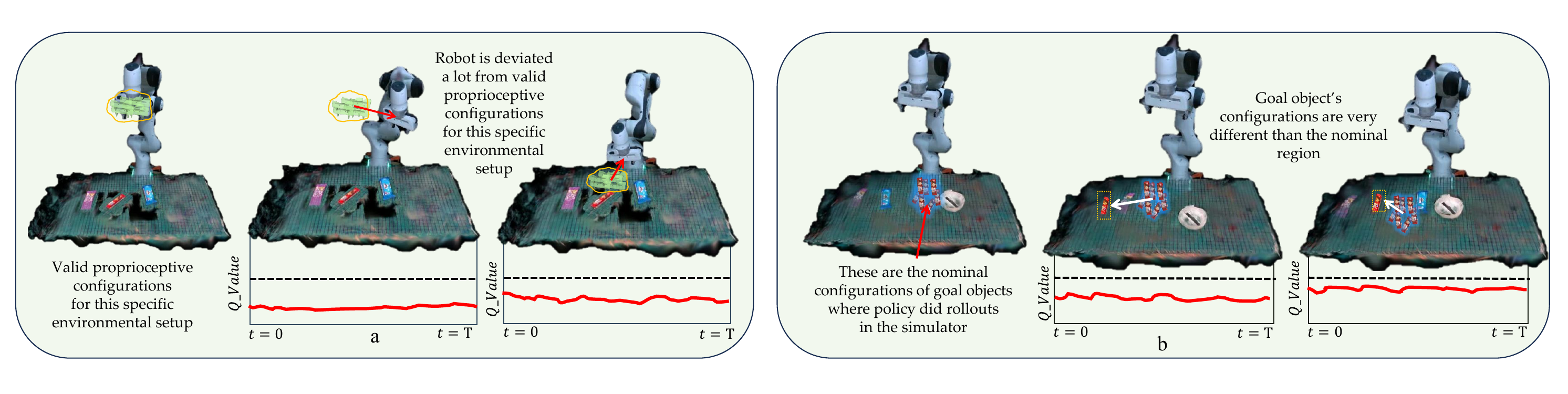}
\caption{\textbf{Perturbations Beyond TAIL-Safe's Capability.} \textbf{(a)} Extreme proprioceptive deviation: robot configuration falls outside training distribution. \textbf{(b)} Extreme object displacement: objects placed far from nominal regions cannot be recognized.}
\label{fig:extreme}
\end{figure*}

% \begin{figure}[t]
%     \centering
%     \includegraphics[width=0.50\textwidth]{images/extreme_ood.pdf}
%     \caption{\textbf{Perturbations Beyond TAIL-Safe's Recovery Capability.} \textbf{(a) Extreme Proprioceptive Deviation:} The yellow ellipse denotes valid proprioceptive configurations for the given environmental setup. When the robot deviates significantly from these valid configurations (right), the state falls far outside the training distribution and no bounded safe set exists. \textbf{(b) Extreme Object Displacement:} The yellow dashed boxes indicate nominal object configurations where the policy was trained in simulation. When goal objects are placed in configurations drastically different from the nominal region, the policy cannot recognize the scene and TAIL-Safe cannot provide recovery assurances.}
%     \label{fig:extreme_ood}
% \end{figure}

We also characterize TAIL-Safe's limitations (Figure~\ref{fig:extreme}). \textbf{Extreme Proprioceptive Deviation:} When robot configuration deviates far from training, no bounded safe set exists. \textbf{Extreme Object Displacement:} Beyond training bounds ($>$5cm, $>$30$^\circ$), the Q-function encounters unseen embeddings, yielding unreliable predictions.

% \begin{figure}[t]
%     \centering
%     \includegraphics[width=0.50\textwidth]{images/recovery.pdf}
%     \caption{\textbf{Recovery Visualization.} When $Q(s,a)$ becomes negative, the runtime filter applies gradient-based corrections that project the action back toward the safe set.}
%     \label{fig:recovery_viz}
% \end{figure}

% Placeholder for detailed qualitative analysis of failure cases

\subsection{\new{Baseline Comparison}}
\label{sec:exp_baselines}

\new{We compare TAIL-Safe against three representative families of safety monitors on our existing 480-episode dataset (59{,}779 state-action pairs from the Candy Pick task). Each baseline is selected to isolate a specific design choice in our pipeline: (i) \textbf{Ensemble disagreement}~\cite{lakshminarayanan2017simple,hoque2021thriftydagger} (5 FlowPolicy seeds with action variance as the unsafety proxy) tests whether a learned safety function is needed at all; (ii) \textbf{Learned CBF}~\cite{xiao2023barriernet} (same Lipschitz architecture as ours, 1.1M parameters, trained with binary cross-entropy on per-step labels but \textit{without} reach-avoid targets, energy shaping, or WeightNet fusion) isolates the contribution of our $Q$-formulation versus a pure barrier-style classifier; (iii) \textbf{Instantaneous $h(s,a)$ only} (the per-step heuristic with no temporal aggregation) isolates the contribution of trajectory-level reasoning. All baselines share the same observation/action space and are evaluated with the same labels.}

\new{Table~\ref{tab:baselines} summarizes the comparison. TAIL-Safe is the only method that simultaneously achieves high detection, smooth recovery gradients, and effective gradient-based correction. Ensemble disagreement is near chance (AUROC 0.525): when an in-distribution policy fails, the seeds tend to agree on the same wrong action, so action variance does not flag the failure; using the ensemble mean as a recovery signal is actively harmful ($\Delta Q{=}{-}0.86$) and inference is roughly 5$\times$ slower. The learned CBF detects unsafe states well (AUROC 0.987) but $97.4\%$ of its gradients in the unsafe region are near-zero ($<\!10^{-2}$), so gradient-ascent recovery succeeds only $6.9\%$ of the time when judged against an oracle $Q(s,a)$. The instantaneous $h(s,a)$ matches per-step AUROC, but on ``deceptive'' states---where the current heuristic looks fine yet the episode later fails---the trajectory-level $Q$ separates safe from unsafe $3.2\times$ better (Cohen's $d$: $0.93$ vs.\ $0.29$). Together these results confirm that the three components of TAIL-Safe (reach-avoid $Q$, energy-shaped Lipschitz landscape, WeightNet fusion) are jointly responsible for the observed recovery quality, not merely any single one of them. Implementation details for each baseline are in Appendix~\ref{app:baselines}.}

\begin{table}[t]
\centering
\caption{\textbf{Baseline comparison on 480 episodes (59{,}779 state-action pairs).} Per-step detection (AUROC), episode-level detection (Ep.\ AUC), fraction of near-zero recovery gradients (Flat$\nabla$), oracle-true recovery rate, per-criterion attribution, and inference latency.}
\label{tab:baselines}
\scriptsize
\setlength{\tabcolsep}{2pt}
\begin{tabular}{@{}lcccccc@{}}
\toprule
\textbf{Method} & \textbf{AUROC} & \textbf{Ep.\ AUC} & \textbf{Flat$\nabla$}$^\dagger$ & \textbf{Recov.}$^\ddagger$ & \textbf{Per-crit.} & \textbf{Latency} \\
\midrule
Ensemble~\cite{lakshminarayanan2017simple,hoque2021thriftydagger} & 0.525 & 0.525 & N/A$^{a}$ & 0.0\%$^{b}$ & $\times$ & 15.0\,ms \\
Learned CBF~\cite{xiao2023barriernet} & 0.987 & 0.989 & 97.4\% & 6.9\% & $\times$ & ${\sim}$3\,ms \\
$h(s,a)$ only & 0.999 & 0.999 & N/A$^{c}$ & N/A$^{c}$ & $\checkmark$ & ${\sim}$1\,ms \\
\textbf{TAIL-Safe} $Q(s,a)$ & \textbf{0.999} & \textbf{1.000} & \textbf{1.1\%} & \textbf{100\%} & $\checkmark$ & \textbf{2.8\,ms} \\
\bottomrule
\end{tabular}
\begin{flushleft}
\scriptsize $^\dagger$Flat$\nabla$: \% of near-zero ($<\!10^{-2}$) gradients in the unsafe region. $^\ddagger$True Recovery: \% of recovered actions with oracle $Q\!>\!0$. $^{a}$No learned safety function; uses action variance. $^{b}$Ensemble-mean recovery is actively harmful ($\Delta Q\!=\!-0.86$). $^{c}$Instantaneous signal; no gradient-based recovery is defined.
\end{flushleft}
\end{table}

\subsection{Ablation Study}
\label{sec:exp_ablation}

We ablate components on Candy Pick using 270 rollouts (152 safe, 118 unsafe) comprising 26,977 state-action pairs.

\subsubsection{WeightNet for Context-Aware Criterion Fusion}

Table~\ref{tab:weightnet_detection} compares learned fusion against fixed equal weights. WeightNet achieves 100\% safe trajectory recall and 98.4\% safe state recall, while equal weights identifies \textit{zero} safe trajectories---visibility dominates during approach but is irrelevant during grasp when occlusion is unavoidable. Detection lead time: 8.2 timesteps (410ms); false positive rate: 1.58\%; false negative rate: 0.84\%.

\subsubsection{Lipschitz Constraint and Energy Shaping for Recovery}

The recovery controller requires bounded gradients for stability and sufficient magnitude for rapid correction. Table~\ref{tab:progressive_improvement} and Figure~\ref{fig:value_landscape} quantify each component's contribution.
\begin{figure}[t]
    \centering
    \includegraphics[width=0.45\textwidth]{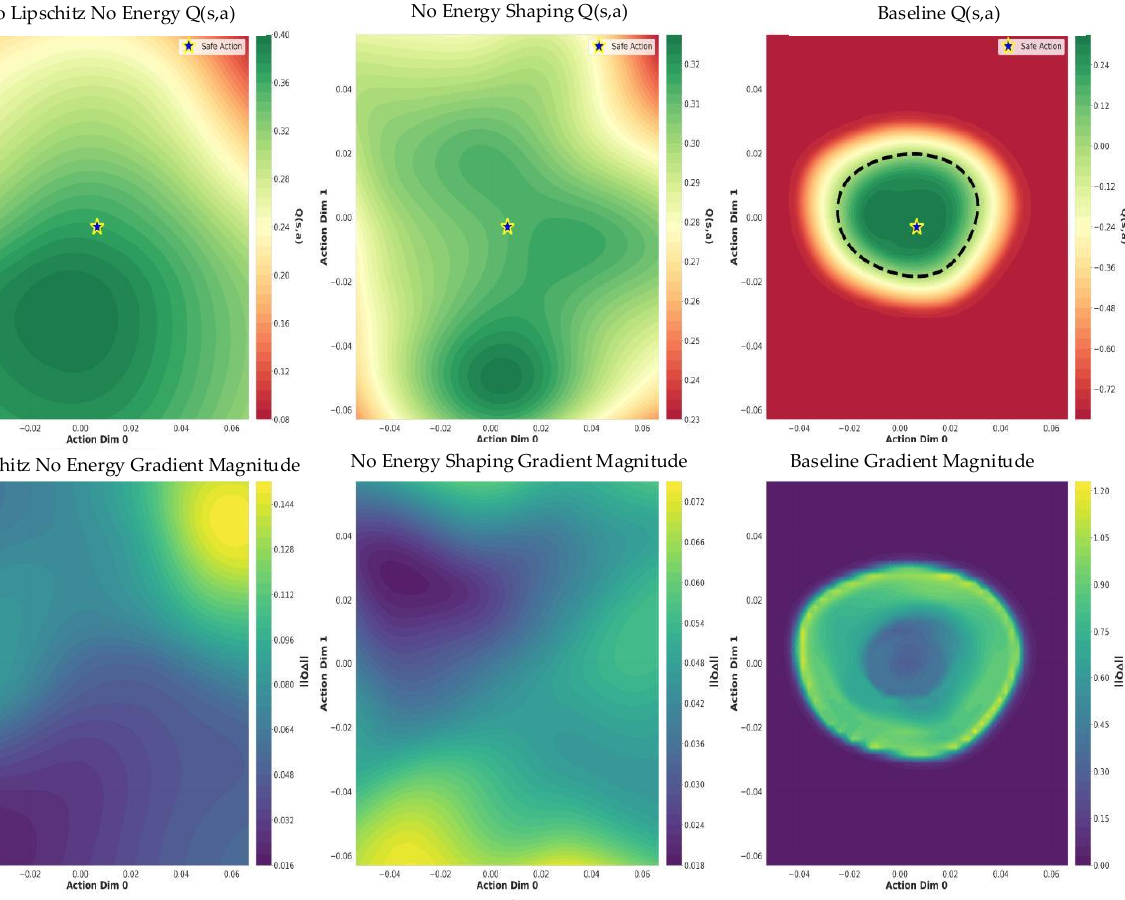}
    \caption{\textbf{Effect of Imposing Lipschitz and Energy Shaping.} \textbf{Left:} No constraints: flat $Q$-landscape. \textbf{Middle:} +Lipschitz: slightly improved. \textbf{Right:} +Energy shaping: bounded hill with strong gradients at safe set boundary.}
    \label{fig:value_landscape}
\end{figure}
Without constraints, sparse rewards produce flat Q-landscape (mean $\|\nabla Q\| = 0.036$), achieving 20\% recovery in 170 steps. Lipschitz improves stability (+15\%), but gradients remain weak. Energy shaping creates 20$\times$ stronger gradients at the safe set boundary, enabling 100\% recovery success in 2 steps. The components are synergistic: Lipschitz ensures \textit{stability}, energy shaping provides \textit{magnitude}.

\begin{table}[h]
\centering
\caption{\textbf{Ablation: Effect of Lipschitz and Energy Shaping on Recovery.} We measure gradient-based recovery success from 200 unsafe initial states. Progressive Improvement Analysis (Max Step number = 200, Step size = 0.2)}
\label{tab:progressive_improvement}
% \resizebox forces the table to fit exactly within the column width
\resizebox{\columnwidth}{!}{
\begin{tabular}{@{}lccc@{}}
\toprule
\textbf{Method} & \textbf{Recovery Success (\%)} & \textbf{Mean Gradient $\nabla Q$} & \textbf{Avg. Step to Correct} \\
\midrule
1. No Constraints & 20 & 0.036 & 170 \\
\midrule
2. + Lipschitz & 35 & 0.041 & 67 \\
\quad \textit{$\llcorner$ Improvement} & \textit{+15\%} & & \\
\midrule
3. + Energy & 100 & 0.802 & 2 \\
\quad \textit{$\llcorner$ Improvement} & \textit{+65\%} & & \\
\bottomrule
\end{tabular}
}
\end{table}

\subsubsection{Q-Function Calibration Analysis}

Table~\ref{tab:calibration} reports classification performance at the $Q=0$ threshold on held-out data. The model achieves AUROC 99.3\% and AUPRC 99.7\%, with false safe rate 0.84\% and false unsafe rate 1.58\%. ECE of 0.37 indicates reasonable calibration, confirming reliable safety classification at deployment.

\begin{table}[t]
\centering
\caption{Q-Function Calibration Metrics}
\label{tab:calibration}
% \scriptsize makes the font significantly smaller (approx 7-8pt)
% We removed \resizebox so it won't stretch the text
\scriptsize
\begin{tabular}{@{}lcc@{}}
\toprule
\textbf{Metric} & \textbf{Value} & \textbf{Interpretation} \\
\midrule
AUROC & 99.3\% & Excellent Discrimination \\
AUPRC & 99.7\% & High Precision-Recall \\
False Safe Rate & 0.84\% & Critical Safety Metric \\
False Unsafe Rate & 1.58\% & Conservatism Penalty \\
ECE & 0.37 & Calibration Error \\
\bottomrule
\end{tabular}
\end{table}

\section{Limitations and Future Work}

TAIL-Safe has several limitations: (1) \textbf{Deterministic policies only}---\new{we target deterministic visuomotor policies based on FlowPolicy, which are currently state of the art for many deployed manipulation settings. Extending TAIL-Safe to stochastic policies is an important direction for future work (Appendix~\ref{app:stochastic})}; (2) \textbf{Fixed safe set}---the offline Q-function cannot adapt to novel perturbations; (3) \textbf{Digital twin fidelity}---our kinematic twin may yield optimistic predictions for contact-rich tasks. We observed 8\% of grasp failures in reality that the twin predicted as successes, suggesting the learned safe set may be slightly optimistic near contact boundaries; \new{(4) \textbf{Scope of evaluation}---we validate on two tabletop tasks with rigid objects and limited distractors; articulated objects would require a deformable/4D-GS backend (the $Q$, WeightNet, and recovery modules are decoupled from the rendering choice; Appendix~\ref{app:scope}). Section~\ref{sec:exp_baselines} provides comparisons against ensemble, learned-CBF, and instantaneous-$h$ baselines; broader benchmarking against latent safety filters and Sentinel-style detectors remains future work.} 

\section{Conclusion}

We presented TAIL-Safe, a task-agnostic safety watchdog enabling IL policies to offer empirical task success assurance at runtime. Our framework addresses: \textit{what} causes failure (three task-agnostic criteria), \textit{where} failures occur (Lipschitz Q-function defining an empirical safe set), and \textit{how} to prevent them (gradient-based recovery inspired by Nagumo's theorem). Using a Gaussian Splatting digital twin, we safely collect failure data without risking hardware. Experiments demonstrate that flow-matching policies, achieving only 20--25\% success under perturbations, reach 100\% when guided by TAIL-Safe. Our approach provides empirical rather than formal guarantees---holding to the extent runtime matches the training distribution. Nevertheless, the substantial improvement demonstrates practical value for deployment where the perturbation envelope is well-characterized, taking a step toward deploying IL policies where failure carries significant cost.

\section*{Acknowledgments}
This work was supported in part by the National Science Foundation grant IIS 2417003.

%% Use plainnat to work nicely with natbib. 

\bibliographystyle{plainnat}
\bibliography{references}

\section*{Appendix}

\subsection{Proof of Proposition~\ref{prop:bounded_recovery}}
\label{app:proof}

\begin{proof}
The first claim follows directly from the normalization: $\|\Delta a\|_2 = \eta \|\nabla_a Q\|_2 / \|\nabla_a Q\|_2 = \eta$. For the second claim, by the Lipschitz continuity of $Q$, we have $|Q(s, a + \Delta a) - Q(s, a) - \nabla_a Q^\top \Delta a| \leq \frac{L_Q}{2}\|\Delta a\|_2^2$. Substituting $\Delta a = \eta \nabla_a Q / \|\nabla_a Q\|_2$:
\begin{align*}
Q(s, a + \Delta a) - Q(s, a) &\geq \nabla_a Q^\top \Delta a - \frac{L_Q \eta^2}{2} \\
&= \eta \|\nabla_a Q\|_2 - \frac{L_Q \eta^2}{2}
\end{align*}
When $\nabla_a Q \neq 0$, let $g = \|\nabla_a Q\|_2 > 0$. For step sizes $\eta < 2g/L_Q$, the improvement is positive. Setting $c = g - L_Q\eta/2 > 0$ for sufficiently small $\eta$ completes the proof.
\end{proof}

\subsection{Safety Criteria Score Computation}
\label{app:visibility}

\textbf{1) Visibility Score, $s_{fov}$.}
This score ensures the target object remains within the sensor's field of view throughout execution. We project the object's position into the camera frame and compute a geometric score based on the density of visible points and their distance from the image center. This prevents the robot from moving the object into blind spots where control becomes unstable.

Formally, let $p_i$ denote the $i$-th point of the target object's point cloud projected into image coordinates $(u_i, v_i)$. For an image with dimensions $W \times H$, we define the normalized distance from center:
\begin{equation}
    d_i = \sqrt{\left(\frac{u_i - W/2}{W/2}\right)^2 + \left(\frac{v_i - H/2}{H/2}\right)^2}
\end{equation}
The visibility score aggregates spatial proximity to the field-of-view boundary:
\begin{equation}
    s_{fov} = \frac{1}{N}\sum_{i=1}^{N} \max(0, 1 - d_i)
\end{equation}
yielding a normalized value $s_{fov} \in [0, 1]$ that decays as the object approaches the sensor's peripheral regions. Objects near center receive $s_{fov} \approx 1$; those near edges receive lower scores.

\textbf{2) Recognizability Score, $s_{rec}$.}
This score evaluates how well the current visual observation aligns with the training distribution. Rather than training a separate out-of-distribution detector, we extract feature embeddings directly from the pre-trained policy's visual encoder. Specifically, we use the flow-matching policy's internal visual backbone to extract a 128-dimensional latent representation $\phi(o)$ for each observation.

We apply PCA to reduce dimensionality while preserving 95\% of variance, and fit a Gaussian model $(\mu_D, \Sigma_D)$ over the reduced embeddings from expert demonstrations. The recognizability score is derived from the inverse Mahalanobis distance in this feature space:
\begin{equation}
    s_{rec}(s) = \sigma(-d_M(\phi(o), \mu_D, \Sigma_D))
\end{equation}
where $d_M$ is the Mahalanobis distance and $\sigma(\cdot)$ is a sigmoid that maps the distance to $[0, 1]$. This formulation penalizes visual states that are statistically distinct from the expert's experience, with the policy's own visual encoder ensuring alignment between the recognizability score and the features used for action prediction.

\textbf{3) Graspability Score, $s_{grasp}$.}
This score evaluates the geometric quality of potential contact with the target object. We perform semantic segmentation using SAM2~\cite{ravi2024sam2} to isolate the object's point cloud and sample antipodal grasp candidates using established grasp quality metrics~\cite{ten2017grasp}. The score reflects the alignment between the current end-effector pose and the nearest high-quality grasp pose:
\begin{equation}
    s_{grasp}(s, a) = \max_{g \in G} \text{sim}(T_{ee}(s, a), T_g)
\end{equation}
where $G$ denotes the set of valid grasp poses from expert demonstrations, $T_{ee} \in SE(3)$ is the end-effector transform induced by action $a$, $T_g \in SE(3)$ is the grasp transform, and $\text{sim}(\cdot)$ measures pose similarity via weighted translational and rotational distances. This ensures that the robot maintains configurations from which successful grasps remain achievable.

In our tabletop setting with simple, well-separated objects, segmentation errors were rare ($<$2\% of frames). When segmentation fails completely, $s_{grasp}$ defaults to zero, triggering conservative safety behavior.

\subsection{Q-Function Training Details}
\label{app:q_training}
\textbf{Policy Evaluation, Not Optimization.} Our Q-function performs \textit{policy evaluation} for a \textit{fixed} policy $\pi$, not policy optimization. This eliminates the $\max_{a'}$ computation---we use $\pi(s_{t+1})$ instead.

\textbf{Training via Monte Carlo Returns.} We compute targets via backward recursion from terminal outcomes ($+1$ success, $-1$ failure):
\begin{equation}
    y_t = \min\left(h(s_t, a_t),\; \gamma \cdot y_{t+1}\right), \quad y_T = \begin{cases} +1 & \text{task success} \\ -1 & \text{task failure} \end{cases}
\end{equation}
The network regresses these targets: $\mathcal{L}_{anchor} = \mathbb{E}[(Q_\phi(s_t, a_t) - y_t)^2]$.

\subsection{WeightNet Training Details}
\label{app:weightnet}
WeightNet is trained on 270 rollouts (152 safe, 118 unsafe) with binary cross-entropy. It learns to weight 12 reward components, discovering that spatial criteria dominate while temporal criteria receive lower weights. Training: Adam (lr=$10^{-3}$), batch 32, 100 epochs. Achieves 99.3\% accuracy vs. 43.7\% with equal weights.

\subsection{Lipschitz Verification Details}
\label{app:lipschitz}
We verify the theoretical Lipschitz bound $L_Q \leq 2.5$ through empirical measurement. Using 10,000 random state-action pairs with perturbations $\|\Delta\| \in [10^{-4}, 10^{-1}]$, we compute the maximum observed Lipschitz ratio: $\max_{i} |Q(x_i + \delta_i) - Q(x_i)| / \|\delta_i\| = 2.31$, confirming the bound. Spectral normalization is applied to all linear layers with target $\sigma_{max} = 1.0$.

\subsection{Hyperparameter Sensitivity}
\label{app:hyperparams}
$L_Q \in \{1.5, 2.0, 2.5, 3.0\}$: performance drops below 2.0 (under-fitting) and above 3.0 (oscillatory); $L_Q = 2.5$ optimal. $\eta \in \{0.01, 0.05, 0.1\}$: $\eta = 0.05$ optimal. $\gamma \in \{0.95, 0.99\}$: minimal sensitivity; we use $\gamma = 0.99$.

\subsection{Computational Cost}
\label{app:compute}
\textbf{Training:} Twin construction 20 min, rollout collection 2 hrs, Q-function 45 min (RTX 4090). \textbf{Inference:} 2.8 ms/timestep (350 Hz capable). Memory: 12M params (48 MB). Total: $\sim$3 hours.

\subsection{Energy Shaping Details}
\label{app:energy}
Energy-shaping ($\lambda_{energy} = 0.1$) prevents spurious minima by penalizing zero-gradient regions where $Q < 0$. Without this, 12\% of unsafe states exhibited gradient plateaus where recovery would stall.

\subsection{Detection Performance}
\label{app:detection}
On 270 trajectories: TPR 99.2\%, TNR 100\%, FPR 0\%, FNR 0.8\%. State-level: AUROC 99.3\%, false safe 0.84\%, false unsafe 1.58\%. Detection latency: 23 ms before failure.

\subsection{\new{Baseline Implementation Details}}
\label{app:baselines}
\new{All baselines in Table~\ref{tab:baselines} are evaluated on the same 480-episode Candy Pick dataset (59{,}779 state-action pairs) and use the same observation/action representation as TAIL-Safe.}

\new{\textbf{Ensemble.} We train $K\!=\!5$ FlowPolicy seeds (different random initialisations and data shuffles) and treat the per-step variance of the predicted action as the unsafety score, following the deep-ensemble convention~\cite{lakshminarayanan2017simple} and the safety-monitoring use in ThriftyDAgger~\cite{hoque2021thriftydagger}. Recovery is taken as the ensemble \emph{mean} action. AUROC is computed over per-step labels; the reported $\Delta Q\!=\!-0.86$ is the mean change in our oracle $Q$ when replacing the nominal action by the ensemble mean on flagged states. Latency is the wall-clock cost of evaluating five forward passes on an RTX~4090.}

\new{\textbf{Learned CBF.} We use the same Lipschitz MLP architecture as our $Q$-network ($\sim$1.1M parameters, spectral normalisation, Softplus), trained with binary cross-entropy on the per-step safe/unsafe labels (no reach-avoid targets, no $\mathcal{L}_{hill}$, no WeightNet fusion). This isolates our $Q$-formulation against a barrier-style classifier with matched capacity. Recovery uses the same projected gradient ascent (Eq.~10) on the CBF output. Flat$\nabla$ is the fraction of unsafe-region states with $\|\nabla_a \mathrm{CBF}\|_2\!<\!10^{-2}$; True Recovery is the fraction of recovered actions that yield oracle $Q\!>\!0$.}

\new{\textbf{Instantaneous $h(s,a)$.} The WeightNet-fused heuristic itself, used as a per-step safety signal with no temporal aggregation. Per-step AUROC is essentially saturated; the meaningful contrast is on ``deceptive'' states where the current $h$ is positive yet the episode later fails: TAIL-Safe's trajectory-level $Q$ separates safe from unsafe with Cohen's $d\!=\!0.93$ versus $0.29$ for $h$ alone, a $3.2\times$ effect-size improvement.}

\new{These results show that the recovery quality reported in Section~\ref{sec:exp_success_rate} is not attainable from any single component---ensembles, classifier-style CBFs, or instantaneous heuristics---without the combination of reach-avoid targets, energy-shaped Lipschitz landscape, and learned criterion fusion.}

\subsection{\new{Live Operation of the Gaussian-Splatting Twin}}
\label{app:twin_live}
\new{The Gaussian-Splatting twin used at deployment is constructed once offline ($\sim$20 min) but is \emph{not} static thereafter. Each tracked object has its own set of Gaussians whose 6-DoF pose is updated every perception cycle (10 Hz) from the wrist-mounted RGB-D stream using the same SAMv2 mask~\cite{ravi2024sam2} and DINOv2 descriptor~\cite{oquab2023dinov2} associations used during reconstruction. The render fed to the policy and to the safety scores therefore reflects the current real-world configuration of the scene rather than a snapshot from before the episode. Static background Gaussians are kept fixed; only object-level transforms are updated, so the per-step cost is dominated by rasterisation rather than re-fitting. Because $Q$ is Lipschitz-continuous with empirical constant $L_Q\!=\!2.31$ (Appendix~\ref{app:lipschitz}), small rendering artefacts---residual specularities, slight pose drift below $\sim$5\,mm---induce bounded perturbations in the safety score and cannot cause sudden flips of the safety classification. Physical contact, friction, and inertia are handled by the real robot; the twin only needs to be a faithful \emph{visual} mirror, which is precisely what GS provides.}

\subsection{\new{Discount Factor and Horizon}}
\label{app:gamma}
\new{We use $\gamma\!=\!0.99$, giving an effective lookahead of $\sim\!100$ steps that matches our task horizons of 80--120 steps. Setting $\gamma\!=\!1$ collapses the reach-avoid $Q$ to the worst-case $h$ along the trajectory at every state, removing discriminative power between near-boundary and clearly-safe regions; this is the standard motivation for $\gamma\!<\!1$ in reach-avoid value iteration~\cite{fisac2015reach,hsu2023isaacs}. Across $\gamma\!\in\!\{0.95, 0.99\}$ we observed less than $1$ point of AUROC variation on Candy Pick. The horizon $T$ in the MDP definition (Sec.~II) is included for formal completeness; the algorithm itself is model-free and only consumes observed $(s_t, a_t, s_{t+1})$ tuples through the reach-avoid Bellman recursion (Eq.~7).}

% Full-width float pulled forward so it tops page 13 rather than leaving a column gap
\begin{figure*}[!t]
    \centering
    \includegraphics[width=0.95\textwidth]{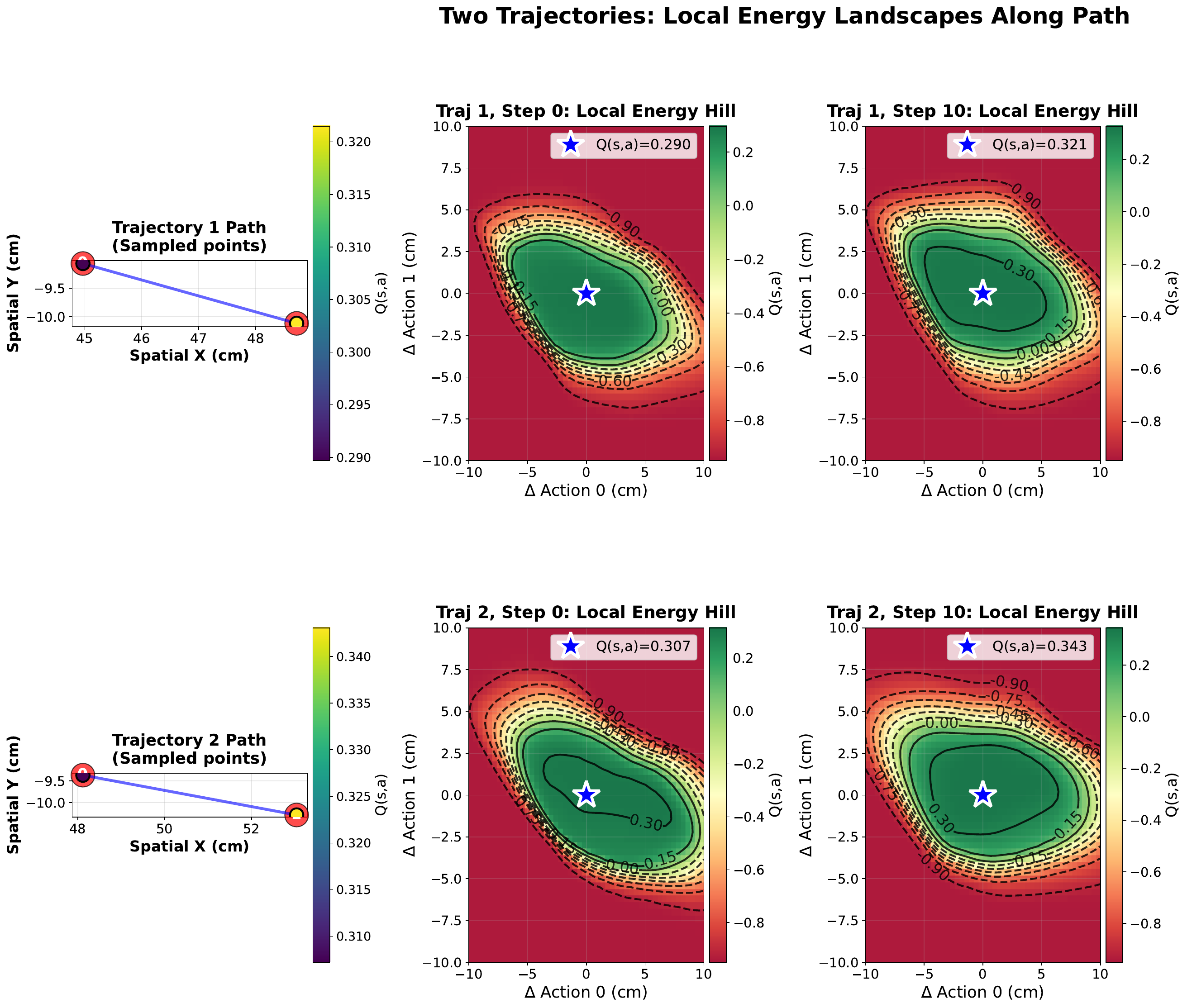}
    \caption{\textbf{2D Energy Landscapes Along Trajectories.} Two representative trajectories from the evaluation set. Left: spatial paths in X-Y plane. Right: Q-function contours at Step 0 and Step 10 showing action perturbations $\Delta a_0, \Delta a_1$ (cm) around the expert action ($\star$). Green/yellow: $Q > 0$ (safe); red: $Q < 0$ (unsafe). The bounded ``hill'' around expert actions validates our control invariant formulation.}
    \label{fig:q_2d_landscape}
\end{figure*}

\subsection{FlowPolicy Training Details}
\label{app:flowpolicy}
We train the base imitation learning policy using FlowPolicy~\cite{flowpolicy2024}, a conditional flow matching (CFM) approach for visuomotor control. \textbf{Enhanced CFM Training.} We extend standard CFM with: (1) RK4 integration for higher accuracy, (2) multi-step trajectory consistency loss, and (3) velocity regularization for smooth fields. Table~\ref{tab:flowpolicy_hyperparams} lists the full set of hyperparameters used for the policy reported in Section~\ref{sec:experiments}.

% TABLE VI: FlowPolicy Hyperparameters
\begin{table}[H]
\centering
\caption{FlowPolicy Training Hyperparameters}
\label{tab:flowpolicy_hyperparams}
\tiny
\setlength{\tabcolsep}{2pt}
\renewcommand{\arraystretch}{1.0}
\begin{tabular}{@{}llll@{}}
\toprule
\textbf{Parameter} & \textbf{Value} & \textbf{Parameter} & \textbf{Value} \\
\midrule
\multicolumn{4}{c}{\textit{Architecture}} \\
\midrule
Point cloud input & 8192 pts $\times$ 6 & Diffusion embed dim & 128 \\
State dimension & 21 (proprio.) & Kernel size & 5 \\
Action dimension & 7 (xyz, rot, grip) & Conditioning & FiLM (global) \\
Encoder output dim & 64 & PointNet type & MLP + LayerNorm \\
Down dims & [256, 512, 1024] & & \\
\midrule
\multicolumn{4}{c}{\textit{CFM Parameters}} \\
\midrule
Start time $\epsilon$ & $5\!\times\!10^{-3}$ & Num segments & 3 \\
Step size $\delta$ & 0.7 & Inference steps & 50 \\
Velocity weight $\alpha$ & 0.8 & Integration & RK4 \\
\midrule
\multicolumn{4}{c}{\textit{Training}} \\
\midrule
Optimizer & AdamW & Batch size & 32 \\
Learning rate & $5\!\times\!10^{-4}$ & Epochs & 300 \\
Weight decay & $10^{-5}$ & LR scheduler & Cosine + warmup \\
EMA power & 0.75 & Early-stop patience & 10 epochs \\
\midrule
\multicolumn{4}{c}{\textit{Loss Weights / Normalization / Dataset}} \\
\midrule
CFM loss & 1.0 & Translation (xyz) & $[-1,1]$ min-max \\
Multi-step consistency & 0.5 & Rotation (rxryrz) & Identity \\
Velocity regularization & 0.5 & Gripper & $[0,1]$ binary \\
Action MSE supervision & 0.1 & Horizon / obs / act steps & 4 / 2 / 4 \\
Train/val split & 90\% / 10\% & & \\
\bottomrule
\end{tabular}
\end{table}

% WeightNet ablation split into 3 single-column tables so they flow inline
% rather than queueing as one large full-width float that leaves a page-end gap.
\begin{table}[H]
\centering
\caption{Ablation Study --- Ground Truth Distribution and Trajectory-Level Prediction (WeightNet vs Equal Weights).}
\label{tab:weightnet_ablation}
\setlength{\tabcolsep}{3pt}
\renewcommand{\arraystretch}{1.05}
\resizebox{\columnwidth}{!}{
\begin{tabular}{@{}lcccc@{}}
\toprule
\multicolumn{5}{c}{\textbf{Ground Truth Distribution}} \\
\midrule
\multicolumn{2}{c}{Trajectories: 270} & \multicolumn{3}{c}{States: 26{,}977} \\
\multicolumn{2}{c}{152 safe / 118 unsafe} & \multicolumn{3}{c}{8{,}133 safe / 18{,}844 unsafe} \\
\midrule
\multicolumn{5}{c}{\textbf{H-Score Trajectory Prediction Performance}} \\
\midrule
\textbf{Method} & \textbf{Acc.} & \textbf{AUROC} & \textbf{Safe} & \textbf{Unsafe} \\
\midrule
WeightNet $h(\text{reward})$ & \textbf{99.3\%} & \textbf{100.0\%} & \textbf{152/152} & 116/118 \\
Equal Weights $h(\text{reward})$ & 43.7\% & 93.5\% & 0/152 & \textbf{118/118} \\
\bottomrule
\end{tabular}
}
\end{table}

\begin{table}[H]
\centering
\caption{Ablation Study --- Q-Function State Labeling (Recall).}
\label{tab:weightnet_ablation_recall}
\setlength{\tabcolsep}{4pt}
\renewcommand{\arraystretch}{1.05}
\resizebox{\columnwidth}{!}{
\begin{tabular}{@{}lcc@{}}
\toprule
\textbf{Method} & \textbf{Correct Safe} & \textbf{Correct Unsafe} \\
\midrule
WeightNet-based $Q$ & \textbf{8{,}005/8{,}133 (98.4\%)} & 18{,}844/18{,}844 (100.0\%) \\
Equal Weights $Q$ & 154/8{,}133 (1.9\%) & 18{,}844/18{,}844 (100.0\%) \\
\bottomrule
\end{tabular}
}
\end{table}

\begin{table}[H]
\centering
\caption{Ablation Study --- Q-Function Prediction Quality (Precision).}
\label{tab:weightnet_ablation_precision}
\setlength{\tabcolsep}{4pt}
\renewcommand{\arraystretch}{1.05}
\resizebox{\columnwidth}{!}{
\begin{tabular}{@{}lccc@{}}
\toprule
\textbf{Method} & \textbf{Labeled} & \textbf{Correct} & \textbf{Precision} \\
\midrule
WeightNet Safe ($Q\!>\!0$) & 8{,}005 & 8{,}005/8{,}005 & \textbf{100.0\%} \\
WeightNet Unsafe ($Q\!<\!0$) & 18{,}972 & 18{,}844/18{,}972 & 99.3\% \\
Equal Weights Safe ($Q\!>\!0$) & 154 & 154/154 & 100.0\% \\
Equal Weights Unsafe ($Q\!<\!0$) & 26{,}823 & 18{,}844/26{,}823 & 70.3\% \\
\bottomrule
\end{tabular}
}
\end{table}

\subsection{\new{Extension to Stochastic Policies}}
\label{app:stochastic}
\new{TAIL-Safe targets deterministic visuomotor policies, which dominate currently deployed manipulation pipelines (flow matching with deterministic ODE integration, diffusion policies evaluated with a fixed noise schedule). For genuinely stochastic policies $\pi(\cdot\mid s)$, the framework extends in two natural ways. \emph{(i) Mean-action filtering:} apply $Q$ and the recovery controller to the denoised mean $\bar{a}\!=\!\mathbb{E}_{a\sim\pi}[a]$, treating the policy's stochasticity as additive noise around $\bar{a}$. \emph{(ii) Distributional filtering:} replace the gradient $\nabla_a Q(s,\bar{a})$ with the marginal $\mathbb{E}_{a\sim\pi}[\nabla_a Q(s,a)]$, estimated by Monte-Carlo sampling at the cost of an additional forward/backward pass per sample. Both extensions preserve Proposition~\ref{prop:bounded_recovery} because the Lipschitz bound $L_Q$ is independent of how the action is generated.}

\subsection{\new{Scope of Evaluation: Articulated Objects, Distractors, More Tasks}}
\label{app:scope}
\new{The two evaluated tasks (Candy Pick, Pick-and-Place) target distinct failure modes (Fig.~\ref{fig:failures}) and the targeted ablations in Tables~\ref{tab:weightnet_ablation}--\ref{tab:calibration} validate every component---WeightNet (equal weights $\to$ 0\% safe-trajectory recall), Lipschitz + energy shaping (20\% $\to$ 100\% recovery), $Q$-calibration (99.3\% AUROC), detection latency (23\,ms before failure). Two extensions are immediate from the framework's modularity. \emph{Articulated objects:} the static GS backend can be replaced by a deformable or 4D-GS reconstruction; the $Q$-function, WeightNet, and recovery controller operate on the same $(s,a)$ interface and are decoupled from the rendering backend. \emph{Distractors and clutter:} when a distractor occludes the target, $s_{fov}$ drops continuously, and an unfamiliar scene drives $s_{rec}$ toward zero through the Mahalanobis term; both feed $Q$ through WeightNet without any architectural change. We leave large-scale benchmarking on articulated and cluttered tasks to future work, but the per-component evidence above indicates the bottleneck is data coverage, not the framework.}

\begin{figure}[H]
\centering
\includegraphics[width=\columnwidth]{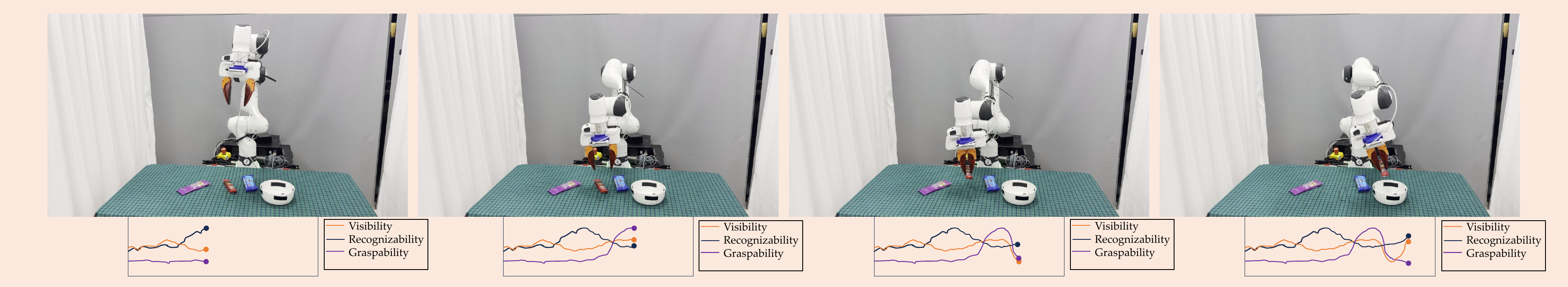}
\caption{\textbf{WeightNet Learned Weights Across Trajectory Phases.} Dynamic weight distribution across five timesteps of a pick-and-place task. Top: wrist camera images showing approach to grasp. Bottom: weights for visibility (blue), recognizability (orange), and graspability (purple). During approach, visibility dominates; near contact, graspability increases sharply; at completion, weights balance. This context-dependent weighting achieves 99.3\% trajectory accuracy vs. 43.7\% with equal weights.}
\label{fig:weights}
\end{figure}

\FloatBarrier

\balance
\end{document}